\documentclass{article}

 \usepackage[preprint]{neurips_2026}

\usepackage[utf8]{inputenc} 
\usepackage[T1]{fontenc}    
\usepackage{hyperref}       
\usepackage{url}            
\usepackage{booktabs}       
\usepackage{amsfonts}       
\usepackage{nicefrac}       
\usepackage{microtype}      
\usepackage{xcolor}         
\usepackage{booktabs}
\usepackage{float}
\usepackage{tabularx}
\usepackage{array}

\usepackage[table]{xcolor}

\usepackage{amsmath}        
\usepackage{amssymb}        
\usepackage{amsthm}         
\usepackage{mathtools}      
\usepackage{bm}             
\usepackage{algorithm}
\usepackage{algpseudocode}

\usepackage{subcaption}

\usepackage{placeins}
\usepackage{float}

\usepackage{enumitem}

\usepackage{tikz,lipsum,lmodern}
\usepackage[most]{tcolorbox}

\theoremstyle{plain}
\newtheorem{theorem}{Theorem}
\newtheorem{lemma}[theorem]{Lemma}
\newtheorem{proposition}[theorem]{Proposition}

\theoremstyle{definition}
\newtheorem{definition}[theorem]{Definition}
 
\theoremstyle{remark}

\newcommand{\R}{\mathbb{R}}

\newcommand{\om}{\omega}
 
\DeclareMathOperator{\CRPS}{CRPS}
\DeclareMathOperator*{\argmin}{arg\,min}   

\newcommand{\Reg}{\mathrm{Reg}}             
\newcommand{\1}{\mathbf{1}}                 

\title{AdaWeather: Adaptively Mixing Probabilistic Weather Forecasts with Logarithmic Regret}

\author{%
  Saptarishi Dhanuka \\
  Ashoka University, India \\
  \texttt{saptarishi.dhanuka@ashoka.edu.in}
  \And
  Sarvesh Iyer \\
  Ashoka University, India \\
  \texttt{sarvesh.iyer@ashoka.edu.in}
  \AND
  Manmeet Singh \\
  Western Kentucky University \\
  \texttt{manmeet.singh@wku.edu}
  \And
  Mihir More \\
  Ashoka University, India \\
  \texttt{mihir.more@ashoka.edu.in}
  \AND
  Rushil Gupta \\
  Ashoka University, India \\
  \texttt{rushil.gupta\_ug2023@ashoka.edu.in}
  \And
  Dhruman Gupta \\
  Ashoka University, India \\
  \texttt{dhruman.gupta\_ug2023@ashoka.edu.in}
  \AND
  Parthasarathi Mukhopadhyay \\
  Ashoka University, India \\
\texttt{parthasarathi.mukhopadhyay@ashoka.edu.in}
  \And
  Sandeep Juneja \\
  Ashoka University, India \\
  \texttt{sandeep.juneja@ashoka.edu.in}
}

\let\oldaddcontentsline\addcontentsline
\renewcommand{\addcontentsline}[3]{}

\usepackage{todonotes}
\usepackage{titlesec}
\begin{document}

\maketitle

\begin{abstract}

Recent advances in machine learning have produced probabilistic weather forecasting models comparable to state-of-the-art numerical weather predictors. But no model consistently dominates spatio-temporally, and relative performance is highly context-dependent. This motivates adaptive methods for combining multiple forecasts to obtain improvements and robustness. While combined forecasts have been proposed in the literature, these are achieved either through supervised learning or through prediction with expert advice methods. We introduce \textbf{AdaWeather}, an adaptive framework that combines many probabilistic forecasts using both machine learning as well as mixture of experts to arrive at a unified improved probabilistic forecast. While traditional expert methods develop the regret bounds with respect to the best single expert in hindsight, we extend the algorithm and analysis to show our method has logarithmic regret compared to the \textit{best static mixture of experts} in hindsight. Empirically, we focus on forecasting temperature, and observe improvements over existing methods.

\end{abstract}

\section{Introduction}

Weather forecasting has critical applications in many fields like agriculture, energy, and disaster management. While numerical weather predictors are used operationally, neural weather models such as FourCastNet \citep{pathak2022fourcastnet}, Pangu-Weather \citep{bi2023pangu}, GraphCast \citep{lam2023graphcast}, AIFS \citep{lang2024aifs}, GenCast \citep{price2024gencast}, NeuralGCM \citep{kochkov2024neuralgcm}, ClimaX \citep{nguyen2023climax}, Aurora \citep{bodnar2024aurora}, and Stormer \citep{nguyen2024stormer} are now comparable to these systems over various metrics, with lesser compute costs. Yet despite architectural inductive biases ranging from graph neural networks, 3D earth-specific attention, spherical Fourier operators, and diffusion-based samplers, benchmarks  \citep{rasp2024weatherbench2} show that these demonstrate remarkably similar aggregate performance, with no single forecaster consistently dominating across variables, lead times, or regimes \citep{chakraborty2025mowe,nai2025geneps,liu2025timefuse}. The natural response is to combine them, as done in operational meteorology \citep{bauer2015quiet}. 

Broadly, these combination or \textit{mixture} methods are either offline or online. An offline algorithm learns a static predictor on historical data while online algorithms produce dynamic forecasts updated at each step based recent forecast errors.
Many existing methods are take the offline track, including classical mixtures of experts \citep{jacobs1991moe,jordan1994hierarchical} and their modern sparse variants \citep{shazeer2017outrageously,fedus2022switch}. These have been adapted to weather as spatially-conditioned MOEs \citep{dryden2022smoe}, variable-adaptive MOEs \citep{chen2025vamoe}, persistence-augmented mixtures \citep{perez2019mixexperts}, and so on \citep{chakraborty2025mowe}. Sample-level fusion \citep{liu2025timefuse} and generative super-ensembles, \citep{nai2025geneps} also contribute a number of powerful architectural backbones \citep{zhou2021informer,wu2021autoformer,zhou2022fedformer,nie2023patchtst,oreshkin2020nbeats,wu2023timesnet}. Such methods are statistically driven because they exploit covariate-dependent structure in model errors, but are frozen at deployment and offer no \textit{guarantees} against distribution shift. \\ Online algorithms mitigate this by reweighing base predictors sequentially as observations arrive \citep{zhang2023onenet,pham2022fsnet,lau2025dsof,fu2022rlmc,liu2025onlinemoe}. We can look at the prediction with expert advice setting where procedures for combining or mixing models or "experts" (called aggregating algorithms), produce optimal results \cite{vovk1990aggregating,vovk1995game,freund1996predicting,littlestone1994weighted,cesa2006prediction}. These are particularly useful when analysing the loss of the mixture against the best possible method, or its \textit{regret}.

Under the assumption that the loss function is \textit{mixable} (e.g. the Continuous Ranked Probability Score or CRPS),  methods inherit logarithmic and controlled regret against shifting comparators in non-stationary environments \citep{aydore2019dynamiclocalregret,herbster1998tracking,hazan2009efficient,jadbabaie2015dynamicregret}. Classical NWP-focused statistical post-processing methods like ensemble model output statistics \citep{gneiting2005emos}, Bayesian model averaging \citep{raftery2005bma}, and proper-scoring-rule calibration \citep{gneiting2007scoring} recalibrate or combine existing forecast outputs using historical forecast–observation data, making them simpler than learned supervised aggregators but less adaptive and theoretically guaranteed than online methods.

Thus a gap exists wherein offline methods exploit historical structure but are unstable under distributional drift, while online aggregators are provably adaptive but do not see structure. To the best of our knowledge, we are aware of no prior work which combines these approaches efficiently, which is filled by us.
We first train a U-Net to learn historical patterns, which then gives us so-called ``side-information'' \cite{cover1996universal} for the aggregating algorithm, when added as an expert alongside other models. We not only get better performance, but also develop a \textit{novel regret bound} against the theoretical best \textit{mixture} in hindsight instead of simply the best expert alone, as done earlier \cite{vyugin2021online, haussler1994tight, vovk1990aggregating}.

\textbf{Key Contributions:}
\begin{enumerate}[topsep=2pt,itemsep=1pt,parsep=0pt,partopsep=0pt]
    \item \textbf{Combined offline-online framework for weather models:} We train a spatio-temporal U-Net aggregation model on historical data and add this model to the existing weather models. We then pass this set of experts to an online aggregation algorithm, which performs the final predictions and achieves improved results when empirically compared against other methods.
    
    \item \textbf{Regret guarantees:} We derive a novel theoretical upper bound on the regret that matches the optimal minimax bound in a special case \cite{freund1996predicting} up to a constant.
\end{enumerate}

In this paper, we first elucidate the related work in Section \ref{sec:related}, formalise the problem in \ref{sec:formulation}, propose the offline-online methods in \ref{sec:method}, provide theoretical analysis in \ref{sec:theory}, and finally present the data, experiments and results, with more proofs and implementation details in the appendices.

\begin{figure}
    \centering
    \includegraphics[width=1.0\linewidth]{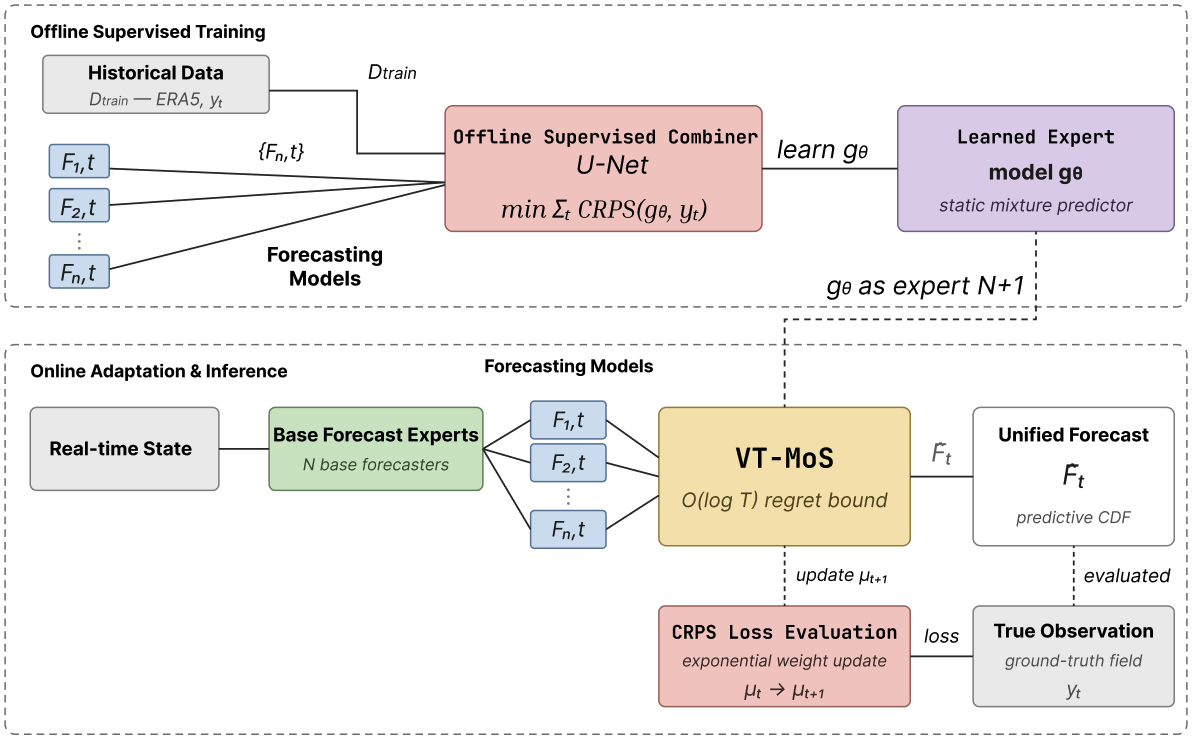}
    \caption{Overall Framework}
    \label{fig:arch}
\end{figure}

{

\section{Related Work}
\label{sec:related}
Various threads relevant to our work are covered below:
\paragraph{Neural weather models:}
Recent AI-based weather models---FourCastNet \citep{pathak2022fourcastnet},
Pangu-Weather \citep{bi2023pangu}, GraphCast \citep{lam2023graphcast},
ClimaX \citep{nguyen2023climax}, Stormer \citep{nguyen2024stormer},
AIFS \citep{lang2024aifs}, Aurora \citep{bodnar2024aurora}, NeuralGCM
\citep{kochkov2024neuralgcm},GenCast, etc
\citep{price2024gencast}---are routinely benchmarked against IFS
\citep{bauer2015quiet} and ERA5 via WeatherBench~2 \citep{rasp2024weatherbench2}.
Despite their architectural diversity, recent benchmarks document near-empirical
convergence and a lack of consistent dominance
\citep{chakraborty2025mowe,nai2025geneps,liu2025timefuse}.
\paragraph{Multi-model combination in weather and time series:}
Operational forecasting has long combined models via grand multi-model means,
super-ensemble regression \citep{krishnamurti1999superensemble}, Bayesian
model averaging \citep{raftery2005bma}, ensemble model output statistics
\citep{gneiting2005emos}, and proper-scoring-rule calibration
\citep{gneiting2007scoring}. The deep-learning era has reframed combination
as a learnable problem. Mixture-of-experts
\citep{jacobs1991moe,jordan1994hierarchical,shazeer2017outrageously,fedus2022switch}
has been adapted with spatially-conditioned routing \citep{dryden2022smoe},
variable-adaptive experts for incremental forecasting \citep{chen2025vamoe},
persistence-augmented mixtures \citep{perez2019mixexperts}, and the
ViT-gated MoWE \citep{chakraborty2025mowe}. Sample-level adaptive fusion
\citep{liu2025timefuse} and generative super-ensembles \citep{nai2025geneps}
extend the design space. The architectural backbones on which these
combiners depend like Informer \citep{zhou2021informer}, Autoformer
\citep{wu2021autoformer}, FEDformer \citep{zhou2022fedformer}, PatchTST
\citep{nie2023patchtst}, N-BEATS \citep{oreshkin2020nbeats}, DeepAR
\citep{salinas2020deepar}, and TimesNet \citep{wu2023timesnet}, closely
interact with distribution-shift normalisation
\citep{kim2021revin,liu2024sand}.
\paragraph{Online forecasting under concept drift:}
OneNet \citep{zhang2023onenet} couples variable-independent and cross-variable
forecasters via online convex programming; FSNet \citep{pham2022fsnet} takes
complementary-learning-systems theory \citep{mcclelland1995cls,kumaran2016cls}
to time series; DSOF \citep{lau2025dsof} introduces dual streams to remove
information leakage in online evaluation; RLMC \citep{fu2022rlmc} treats
dynamic weighting as reinforcement learning; and \cite{liu2025onlinemoe}
propose an online MOE with no-regret guarantees. The continual-learning
machinery these methods share \citep{lopezpaz2017gem} connects to the broader
concept-drift literature \citep{tsymbal2004drift,gama2014drift}.
\paragraph{Prediction with expert advice:}
The theoretical foundation traces to \cite{vovk1990aggregating,littlestone1994weighted} and \cite{haussler1994tight}, which established
tight $\ln N$ regret bounds for mixable losses. Then, \cite{freund1996predicting}
unified the analysis through exponential weights. We refer the reader to \cite{cesa2006prediction,shalev2012online,hazan2016introduction} for further work which consolidates the field. The mixability of CRPS was established by \cite{dombry2024crps} and exploited online by \cite{vyugin2021online}. In atmospheric science, expert aggregation has been applied to temperature combination over France \citep{pfitzner2025ea} via EWA and
MLpoly, but only to limited sparse station data with NWP forecasts, and without an offline mixing stage.
}

\section{Problem Formulation}
\label{sec:formulation}

\subsection{Preliminaries}
\label{subsec:setting}

\paragraph{Forecasts and observations.}
We work in the probabilistic forecasting setup with multiple forecast
trajectories generated for each day by each model. Index forecast issue
times by $t \in [T]$ and lead times by $\tau \in \mathcal{T}$. For each
pair $(t,\tau)$, we observe a target field $y_{t,\tau} \in \mathcal{Y}
\subseteq \R^{|\mathcal{G}|}$ on a spatial grid $\mathcal{G}$, with each
grid-point value lying in a known interval $[a,b] \subset \R$.
Boundedness is realistic for variables such as temperature and is
essential for the bounds we derive.
There are $N$ base forecasters, or \emph{experts}. For each $(t,\tau)$,
expert $n$ outputs a predictive CDF $F_{n,t,\tau}$ over $\mathcal{Y}$,
derived from an atmospheric initial state $x_t$ via a model pass
$f_n : x_t \mapsto F_{n,t,\tau}$. In practice $F_{n,t,\tau}$ is
represented by an $M$-member ensemble and we work with the empirical
CDF; our analysis imposes no structure on expert predictions.  
\paragraph{Loss.}
Quality is measured by the \emph{continuous ranked probability score}
(CRPS) \citep{gneiting2007scoring}, evaluated pointwise and averaged
over $\mathcal{G}$. For any predictive CDF $\hat F$ and observation $y$
at a single grid point,
\begin{equation}
  \CRPS(\hat F, y)
  \;=\; \int_{a}^{b}\big(\hat F(z) - \1\{z \ge y\}\big)^{2}\, dz.
  \label{eq:crps}
\end{equation}
For an empirical CDF given by an ensemble $\{e_1,\dots,e_M\}$, CRPS
admits the closed form
$\tfrac{1}{M}\sum_i |e_i - y|
 - \tfrac{1}{2M^2}\sum_{i,j}|e_i - e_j|$
(this follows from arguments analogous to those in
Lemma~\ref{lem:quadratic}); in implementation we use the unbiased
$M(M{-}1)$-normalised plug-in estimator (details in
Appendix~\ref{app:crps}). We let $\CRPS$ act on
grid-valued $(\hat F_{t,\tau}, y_{t,\tau})$ by averaging \eqref{eq:crps}
over $g \in \mathcal{G}$, and write the cumulative loss as
$$\mathcal{L}_T = \sum_{t=1}^{T}\sum_{\tau \in \mathcal{T}}
 \CRPS(\hat F_{t,\tau}, y_{t,\tau}).$$ 
CRPS is a strictly proper scoring rule \citep{gneiting2007scoring} and
is $\tfrac{2}{b-a}$-mixable on $[a,b]$ \citep{vyugin2021online,
dombry2024crps}. It also carries a quadratic structure
(Lemma~\ref{lem:quadratic}) that is central to controlling its
fluctuations. Both properties drive the logarithmic regret bound of
Section~\ref{sec:vtmos}.

\subsection{Aggregation algorithms}
\label{subsec:aggregation}
An aggregation algorithm $\mathcal{A}$ produces a predictive CDF
$\hat F_{t,\tau}
 = \mathcal{A}\!\big(x_t, \{F_{n,t,\tau}\}_{n=1}^{N}\big)$
from the initial state and the expert predictions; see
\cite{vovk1995game},\cite{vyugin2021online} for canonical examples. We consider both offline and online variants.

\paragraph{Offline aggregation.}
Following typical supervised learning, given a training split $\mathcal{D}_{\mathrm{train}}
 = \{(x_t, \{F_{n,t,\tau}\}_n, y_{t,\tau}) : t \le T_0\}$, an offline
aggregator $\hat F_{t,\tau} = g_\theta(x_t, \{F_{n,t,\tau}\}_n)$ is fit
by empirical risk minimisation,
\begin{equation}
  \theta^\star \;\in\; \arg\min_\theta\;
  \sum_{t \le T_0}\sum_{\tau \in \mathcal{T}}
  \CRPS\!\big(g_\theta(x_t, \{F_{n,t,\tau}\}_n),\, y_{t,\tau}\big).
  \label{eq:erm}
\end{equation}

\paragraph{Online aggregation.}
At each step the algorithm plays a CDF $\hat F_t$
and incurs loss $\CRPS(\hat F_t, y_t)$, then updates once $y_t$ is
revealed. A canonical family of online aggregators is the
\emph{linear mixture} indexed by weights $w \in \Theta := \{w \in
\R_{\ge 0}^{N} : \sum_n w_n = 1\}$,
$\hat F^{(w)}_{t,\tau}=\sum_{n=1}^{N} w_n\, F_{n,t,\tau}$
Our key methodological contribution is a second online procedure
producing a forecast that is non-linear in the expert predictions
\emph{and} the weights (Section~\ref{sec:vtmos}).

\paragraph{Regret.}
We measure online performance by the static regret against the best
fixed convex combination in hindsight,
\begin{equation}
  R_T
  \;=\; \sum_{t=1}^{T}\CRPS(\hat F_t, y_t)
  \;-\; \min_{w^\star \in \Theta}
        \sum_{t=1}^{T}\CRPS\!\big(\hat F^{(w^\star)}_t, y_t\big).
  \label{eq:regret}
\end{equation}
This is strictly stronger than regret against the best single expert,
since $\Theta$ contains the vertices $e_n$. For best-expert regret on
mixable losses, an algorithm utilising mixability achieves an optimal
$O(\ln N)$ bound independent of $T$
\citep{vovk1995game,haussler1994tight,freund1996predicting,
cesa2006prediction}; for convex losses the rate degrades to
$O(\sqrt{T \ln N})$
\citep{shalev2012online,hazan2016introduction}. In
Section~\ref{sec:vtmos} we derive an aggregator whose regret
\eqref{eq:regret} grows $ln(t)$ with an explicit leading
constant; we discuss optimality of this constant in
Section~\ref{app:regproof} via comparison to the two-expert
setting in \cite{freund1996predicting}.

\subsection{Problem statement}
\label{subsec:problem}
Hence, given $N$ base forecasters and $\mathcal{D}_{\mathrm{train}}$, we seek
an aggregation algorithm $\mathcal{C}$ that simultaneously
\begin{enumerate}
  \item exploits the dependent structure exposed by
        $\mathcal{D}_{\mathrm{train}}$, and
  \item admits a sublinear regret bound of the form \eqref{eq:regret}
        on the evaluation horizon $t > T_0$.
\end{enumerate}
Purely offline ERM addresses (a) but loses all guarantees under shift,
while a purely online aggregator over $\{F_{n,t}\}$ addresses (b) but
discards every error pattern visible in $\mathcal{D}_{\mathrm{train}}$.
We resolve both by feeding an offline-trained expert as an additional
input to an online aggregator.

\section{Method}
\label{sec:method}

\subsection{Offline supervised training}
\label{subsec:offline-method}

Our offline aggregator is a U-Net that maps expert forecasts to a
probabilistic mixture forecast over an $H \times W$
India grid at $L$ lead times. It learns which forecast models to weigh more based on historical data. More details are in Appendix~\ref{app:offline}.

\subsection{Online adaptation}
\label{subsec:online-method}

At inference time we treat the trained offline model as an additional
expert and combine it online with the other $N$ forecasters. We suppress the lead-time
index $\tau$ throughout the analysis but all results apply uniformly
across lead times. We
describe two online aggregators: Vovk's algorithm as a baseline, covered in Appendix \ref{sec:backgroundexperts} and
our proposed VT-MOS. Vovk achieves $O(\ln N)$ regret against the best
single expert \citep{vovk1995game,haussler1994tight}, but degrades
against the strictly stronger best-mixture comparator
\eqref{eq:regret}: weights are maintained on individual experts and
not on their mixtures, while the non-linearity of CRPS means a mixture
can substantially outperform every individual expert (see
Lemma~\ref{lem:quadratic}). Lower bounds of $\Omega(\sqrt{T})$ are
known for related losses \citep{cesa1997use}. This motivates an
algorithm that tracks weights over mixtures directly.

\subsubsection{VT-MOS: aggregation for mixtures of experts}
\label{sec:vtmos}

Our proposed algorithm, Vyugin-Trunov Mixture of Simplex (VT-MOS), adapts the Vyugin--Trunov aggregator
\citep{vyugin2021online} to the mixture-of-experts comparator. Write
$F_{t}^{(p)} = \sum_{n=1}^{N} p_n F_{n,t}$ for the predictive CDF
induced by mixture $p \in \Theta$. The mixability framework underlying
the aggregator needs to be tailored to mixtures.

\begin{definition}\label{def:mix}
For $\eta > 0$, the CRPS loss is \emph{$\eta$-mixable in the MOE
setting} if, for every collection of expert CDFs $\{F_n\}_{n=1}^N$ and
every finite measure $\nu$ on $\Theta$, there exists a CDF $F$ such
that for all $y \in [a,b]$,
\begin{equation}
  \CRPS(F, y)
  \;\le\; -\frac{1}{\eta}\,
  \ln \frac{\int_\Theta
    \exp\!\big(-\eta\,\CRPS(F^{(p)}, y)\big)\,d\nu(p)}{\nu(\Theta)}.
  \label{eq:mixineq}
\end{equation}
\end{definition}
VT-MOS maintains a measure $\mu_t$ on $\Theta$, initialised to the
uniform prior $\mu_0$ (a choice supported by the two-expert analysis
of \citealp[Sections 4,7]{freund1996predicting}). At round $t$ it
outputs a CDF $\hat F_t$ satisfying \eqref{eq:mixineq} relative to
$\mu_t$, then updates
\begin{equation}
  \frac{d\mu_{t+1}}{d\mu_t}(p)
  \;=\; \exp\!\big(-\eta\,\CRPS(F_{t}^{(p)}, y_t)\big).
  \label{eq:update}
\end{equation}
Theorem~\ref{thm:mixability} (Section~\ref{subsec:mixability})
establishes that such an $\hat F_t$ exists in closed form, with $\eta = 2/(b-a)$.
\paragraph{Closed form.}
Substituting the explicit form of $\hat F_t$ (Theorem~\ref{thm:mixability})
into \eqref{eq:update} and unrolling, the prior $\mu_t$ collapses into
the integrand and $\hat F_t$ depends only on the expert predictions and
observations seen so far:
\begin{equation}
  \hat F_t(y) \;=\; \tfrac{1}{2}
  \;-\; \tfrac{1}{4}\,\ln\!\left(
    \frac{\int_\Theta \exp\!\big(-2\big\{[F_t^{(p)}(y)]^2
                                          + S_{t-1}(p)\big\}\big)\,dp}
         {\int_\Theta \exp\!\big(-2\big\{[1 - F_t^{(p)}(y)]^2
                                          + S_{t-1}(p)\big\}\big)\,dp}
  \right),
  \label{eq:Ft_explicit}
\end{equation}
where $S_{t-1}(p) = \sum_{j=1}^{t-1} \CRPS(F_{j}^{(p)}, y_j)$ and $dp$
is Lebesgue measure on $\Theta$.

\paragraph{Monte Carlo implementation.}
The integrals over $\Theta$ are intractable in closed form but admit a
simple Monte Carlo estimator. We sample $p \sim \mathrm{Unif}(\Theta)$
via the standard exponential representation: draw
$X_{nj} \stackrel{\mathrm{iid}}{\sim} \mathrm{Exp}(1)$ for
$n \in [N], j \in [\mathsf{M}]$ and set
$v_{nj} = X_{nj}/\sum_{n'} X_{n'j}$. We reuse the same samples
$\{v_{\cdot j}\}$ in both numerator and denominator: the integrands
differ only in the leading squared term, so the shared $S_{t-1}(p)$
factor and correlated draws yield substantial variance reduction. This
gives Algorithm~\ref{alg:vtmos}.

\begin{algorithm}[H]
\caption{VT-MOS: variance-reduced Vyugin--Trunov for mixtures of
experts}
\label{alg:vtmos}
\begin{algorithmic}[1]
\Require Number of experts $N$, support bounds $a < b$, Monte Carlo
size $\mathsf{M}$.
\State Sample $X_{nj}
   \stackrel{\mathrm{iid}}{\sim} \mathrm{Exp}(1)$ and set
   $v_{nj} \gets X_{nj}/\sum_{n'} X_{n'j}$ for
   $n \in [N], j \in [\mathsf{M}]$.
\State Initialise $S_j \gets 0$ for $j \in [\mathsf{M}]$.
   \Comment{Cumulative loss of mixture $v_{\cdot j}$.}
\For{$t = 1, \ldots, T$}
  \State Receive expert predictions $\{F_{n,t}\}_{n=1}^{N}$.
  \State \textbf{Output} $\hat F_t(\cdot)$ defined for
    $y \in [a,b]$ by
  \[
    \hat F_t(y) \;=\; \tfrac{1}{2} \;-\; \tfrac{1}{4}\,\ln\!\left(
      \frac{\sum_{j=1}^{\mathsf{M}}
        \exp\!\big(-2\big[(\sum_n v_{nj} F_{n,t}(y))^2 + S_j\big]\big)}
           {\sum_{j=1}^{\mathsf{M}}
        \exp\!\big(-2\big[(1 - \sum_n v_{nj} F_{n,t}(y))^2
                        + S_j\big]\big)}
    \right).
  \]
  \State Receive observation $y_t$.
  \State $S_j \gets S_j
    + \CRPS\!\big(\sum_n v_{nj} F_{n,t},\, y_t\big)$
    for each $j \in [\mathsf{M}]$.
\EndFor
\end{algorithmic}
\end{algorithm}

\section{Theoretical Analysis}
\label{sec:theory}

We now establish the two main theorems behind VT-MOS: that CRPS is
mixable in the MOE sense (Theorem~\ref{thm:mixability}, justifying
the closed form \eqref{eq:Ft_explicit}), and that the resulting
algorithm enjoys logarithmic regret with an explicit leading constant
(Theorem~\ref{thm:reg}). Throughout, we measure performance by the
\emph{mixture-of-experts regret}
\begin{equation}
  \Reg_{T}
  \;=\; \sum_{t=1}^{T} \CRPS(\hat F_t, y_t)
       \;-\; \min_{p \in \Theta}
             \sum_{t=1}^{T} \CRPS\!\big(F_{t}^{(p)}, y_t\big),
  \label{eq:mos_regret}
\end{equation}
which, since $\Theta$ contains the vertices $e_n$, is at least as large
as the standard best-expert regret.

\subsection{MOE-mixability of CRPS}
\label{subsec:mixability}

\begin{theorem}[Mixability]\label{thm:mixability}
The CRPS loss is $\tfrac{2}{b-a}$-mixable in the MOE setting. Moreover,
inequality \eqref{eq:mixineq} is satisfied by the explicit CDF
\begin{equation}
  \hat F(y) \;=\; \tfrac{1}{2}
  \;-\; \tfrac{1}{4}\,\ln\!\left(
    \frac{\int_\Theta \exp\!\big(-2\big[F^{(p)}(y)\big]^2\big)\,d\nu(p)}
         {\int_\Theta \exp\!\big(-2\big[1 - F^{(p)}(y)\big]^2\big)
                                                    \,d\nu(p)}
  \right).
  \label{eq:Fhat}
\end{equation}
\end{theorem}

\paragraph{Proof sketch.}
The argument follows the structure of \cite[Theorem 2]{vyugin2021online}
with adaptations at two key points. First, discretise $[a,b]$ at scale
$\Delta = (b-a)/d$ and approximate
$\CRPS(F, y) \approx \Delta \sum_{s=1}^{d}(F(z_s) - \1\{z_s \ge y\})^2$
up to $O(\Delta)$ error. This step is standard. Second, where the
single-expert proof applies mixability bin-by-bin to a finite expert
set, we apply it bin-by-bin against the measure $\nu$ on $\Theta$,
producing $\hat f_s = \hat F(z_s)$, and then recombine bins via
H\"older's inequality with exponents aligned at $\eta = 2/(b-a)$. As
noted in \cite[p.~7]{freund1996predicting}, this amounts to replacing
sums over experts by integrals over $\Theta$ with appropriate
adjustments to the H\"older space. Letting $d \to \infty$ completes
the proof; full details are in Appendix~\ref{app:mixproof}. The rate $\eta = 2/(b-a)$ shrinks as the support widens, consistent
with the intuition that prediction is harder on a larger range.

\subsection{Regret bound}
\label{subsec:regret}

\begin{theorem}[Regret bound]\label{thm:reg}
For VT-MOS run with exact integration over $\Theta$
($\mathsf{M} \to \infty$),
\[
  \Reg_T
  \;\le\; \frac{(b-a)(N-1)}{2}\,\ln T \;+\; C,
\]
where $C$ depends only on $a, b, N$ and is independent of $T$.
\end{theorem}
The bound is logarithmic in $T$ and grows linearly in $N - 1$ and the
support length, both of which match the intuition that larger
comparator classes and wider observation ranges make prediction
harder. The proof rests on two structural facts about CRPS.

\begin{lemma}[Quadratic structure]\label{lem:quadratic}
Let $L(p) = \sum_{t=1}^{T} \CRPS(F_t^{(p)}, y_t)$, and define
$A_i = \sum_{t} \CRPS(F_{i,t}, y_t)$ and
$B_{ij} = \sum_{t} \int_a^b (F_{i,t}(z) - F_{j,t}(z))^2\, dz$. Then
$L(p) = A^\top p - \tfrac{1}{2}\, p^\top B p$ for all $p \in \Theta$.
\end{lemma}

\begin{lemma}[Worst-case bounds]\label{bds}
For all $1 \le i, j \le N$ and $p \in \Theta$:
\textup{(i)}~$|A_i - A_j| \le 2(b-a)T$;
\textup{(ii)}~$|(Bp)_i - (Bp)_j| \le (b-a)T$;
\textup{(iii)}~$-\delta^\top B \delta \le (b-a)NT \|\delta\|_2^2$ for
every $\delta$ with $\mathbf{1}^\top \delta = 0$.
\end{lemma}

These bounds are worst-case and can be tightened under structural
assumptions on the experts; see Section~\ref{app:regproof}.

\paragraph{Proof sketch.}
Iterating \eqref{eq:mixineq} and \eqref{eq:update} yields the standard
loss identity
\begin{equation}
  \sum_{t=1}^{T} \CRPS(\hat F_t, y_t)
  \;\le\; \eta^{-1}\,\ln \frac{\nu_0(\Theta)}{\mu_{T+1}(\Theta)}
  \;=\; \eta^{-1}\,\ln \frac{1}{\mu_{T+1}(\Theta)},
  \label{eq:lossmin}
\end{equation}
using $\mu_0 \equiv 1$. Let $p^\star \in \arg\min_\Theta L$, which
exists by continuity on the compact simplex. Combining
\eqref{eq:lossmin} with Lemma~\ref{lem:quadratic} yields
\begin{equation}
  \Reg_T \;\le\;
  -\tfrac{b-a}{2}\,\ln \int_{\Theta}
  \exp\!\Big(-\tfrac{2}{b-a}\big[L(p) - L(p^\star)\big]\Big)\, dp.
  \label{eq:reg_integral}
\end{equation}
Setting $\delta = p - p^\star$ and using
$L(p) - L(p^\star) = (A - Bp^\star)^\top \delta
 - \tfrac{1}{2}\,\delta^\top B \delta$, Lemma~\ref{bds} bounds the
linear term by $3(b-a)T\sqrt{N}\,\|\delta\|_2$ and the quadratic term
by $(b-a)NT\,\|\delta\|_2^2$. Choosing
$r = 1/\big(T(3\sqrt{N} + N)\big)$ gives
\[
  \|\delta\|_2 \le r
  \;\implies\; L(p) - L(p^\star) \le b - a
  \;\implies\; \exp\!\big(-\tfrac{2}{b-a}[L(p) - L(p^\star)]\big)
              \ge e^{-2}.
\]
The integral in \eqref{eq:reg_integral} is therefore lower-bounded by
$e^{-2}\,|\Theta \cap B(p^\star, r)|$. Since $\Theta$ is a regular
$(N-1)$-simplex, this volume is at least $c_\Theta\, r^{N-1}$ for a
constant $c_\Theta$ depending only on $N$, yielding
\[
  \int_{\Theta}
  \exp\!\Big(-\tfrac{2}{b-a}\big[L(p) - L(p^\star)\big]\Big)\, dp
  \;\ge\; C\, T^{-(N-1)}.
\]
Substituting into \eqref{eq:reg_integral} gives the claimed bound.
Full details are in Appendix~\ref{app:regproof}.

\section{Data and Experiments}

\subsection{Datasets}

Due to computational reasons and storage constraints, we restrict our main experiments to 0.25$^{\circ}$ 2m temperature data over India from 2019-2026 and 3 day forecasts in intervals of 12 hours (to match the intervals of GenCast). More details about the datasets and preprocessing are described in Appendix \ref{app:data}. The idea can be applied to any general bounding box, time period or variable. We train the multi-model combiner from 2019-2022, validate over 2023, and test over 2024-2025 over ERA5 reanalysis data. Another reason for this 2019 cutoff is that the most recent data-driven forecast (GenCast) is trained uptil 2018, and FGN is till 2022. Hence, forecasts generated by these models prior to this period will be "testing on training data" which is unsound due to data leakage.

\subsection{Experiments}

Full experimental details are in Appendix \ref{app:expts}.

We train our supervised U-Net model on the forecast-ERA5 ground truth pairs from 2019-2022, and evaluate on 2024-2025 years. The details of the U-Net are in Appendix \ref{app:offline}. 
We evaluate our mixture method against individual forecasting models, equal weighting, online aggregators, and offline mixture baselines using CRPS. We also analyze regret against the best forecast in hindsight and, on selected subregions where optimization is tractable, the best convex combination of forecasts in hindsight. \\ During evaluation, we feed the U-Net prediction as an expert into the online expert-advice methods. The task is to aggregate $N=5$ heterogeneous ensemble forecast systems into a single probabilistic forecast for 2-meter temperature (T2m) over India.

\textbf{Baseline and regret benchmark}

\emph{Equal Weight} uniformly averages the $N$ available ensemble forecasts.
As an offline lower bound, we compute the \emph{best-in-hindsight} (BIH)
static convex combination
$w^\star(\ell)\in\Delta^{N-1}$
that minimizes the closed-form fair-CRPS over the entire test window for
each lead. This is solved using a per-lead simplex-constrained SLSQP
optimization over the $(a,B)$ coefficients:
$$
\mathrm{CRPS}(w)
\approx
\sum_t
\sum_{\text{cell}}
\left[
a(t,\ell,\cdot,\text{cell})^\top w
-
\frac12
w^\top
B(t,\ell,\text{cell})
\,w
\right].
$$
No online combiner can outperform this oracle in expectation; the gap to
BIH defines the \emph{regret}.

\subsection{Evaluation protocol}
For each method $m$, lead $\ell$, and initialization $t$, we evaluate CRPS
in two ways: (i) \emph{patch CRPS} over a $5\times5$ window centered on
Delhi and other Indian cities, and (ii) the spatial mean over the full
India grid. We evaluate all methods using cumulative-regret curves relative
to the best-in-hindsight (BIH) oracle, expert-weight contribution dynamics,
per-city single-pixel CRPS across major Indian cities, and India-wide
spatial averages. We additionally generate spatial improvement maps and
best-method-per-pixel visualizations relative to the Equal-Weight baseline.
For hybrid methods, we analyze the temporal trade-off between direct expert
weighting and U-Net trust through the $(N{+}1)$-band decomposition of expert
contributions. We retrain the MoWE framework over our experts with a CRPS loss as a benchmark \cite{chakraborty2025mowe}.

\section{Results}

\FloatBarrier
We observe that both, the trained U-Net and the VT-MOS algorithm, perform better than all individual experts when considering the average performance over lead times. Combining the two by using the trained U-Net as an expert to the VT-MOS algorithm leads to the best results across methods, and admits the least cumulative regret as well.

\begin{figure}
    \begin{subfigure}{0.48\textwidth}
        \centering
        \includegraphics[width=\linewidth]{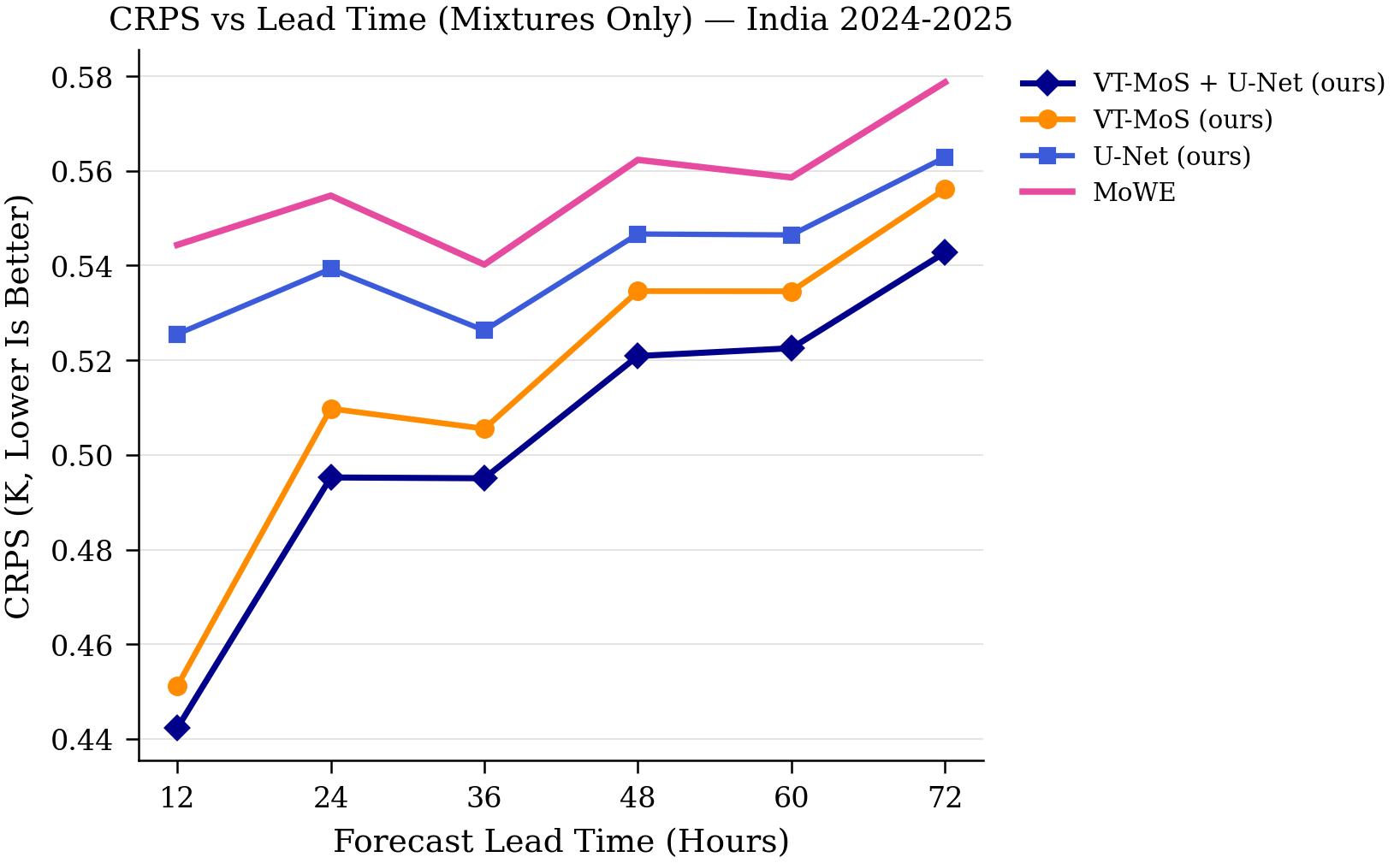}
        \caption{CRPS (lower better) across lead times for evaluation period}
        \label{fig:crpslead_mix}
    \end{subfigure}
    \hfill
    \begin{subfigure}{0.48\textwidth}
        \centering
        \includegraphics[width=\linewidth]{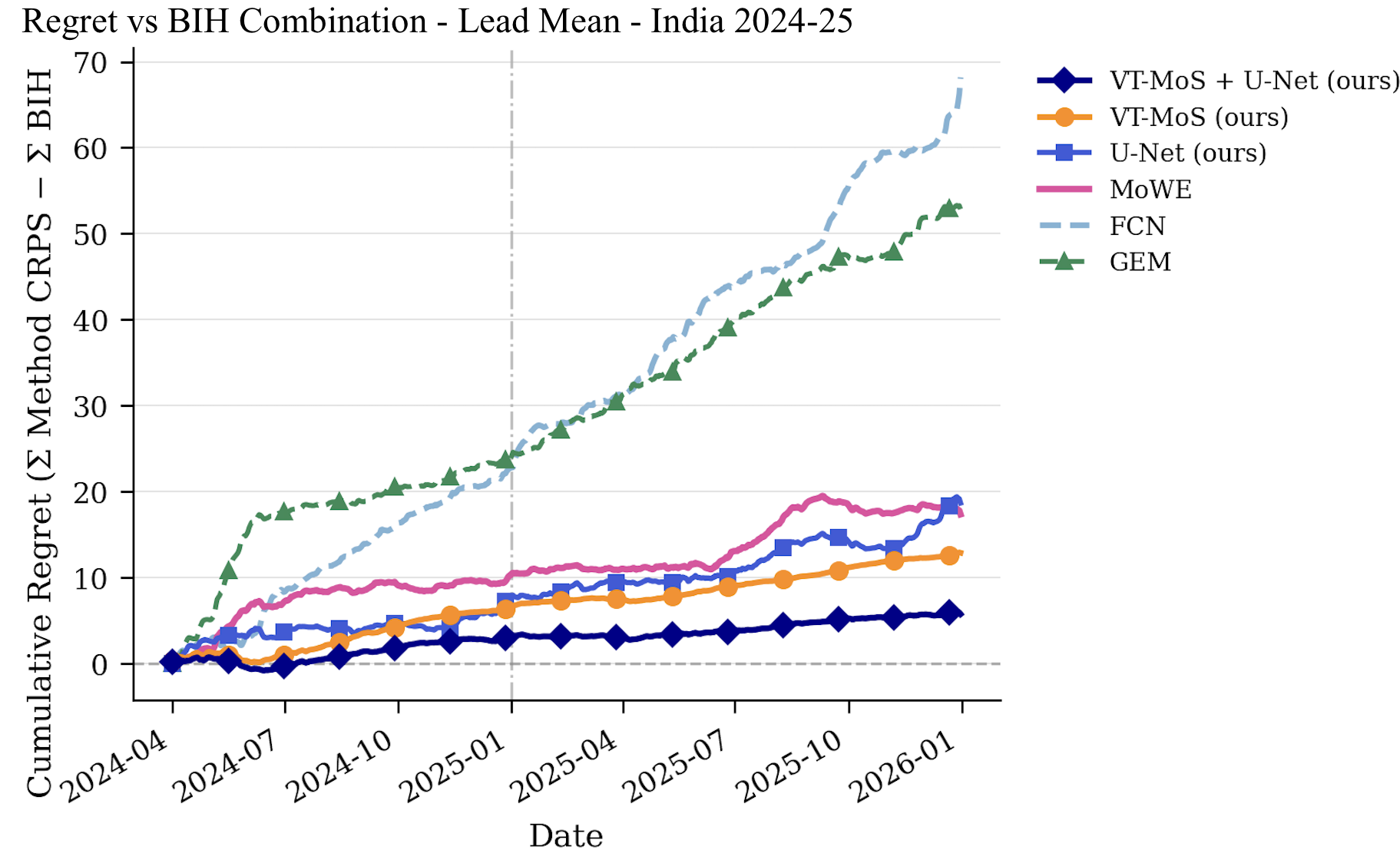}
        \caption{CRPS regret (2 individual experts for contrast) against Best-in-Hindsight (BIH)}
        \label{fig:bih_regret_lead}
    \end{subfigure}
    \begin{subfigure}{0.48\textwidth}
        \centering
        \includegraphics[width=\linewidth]{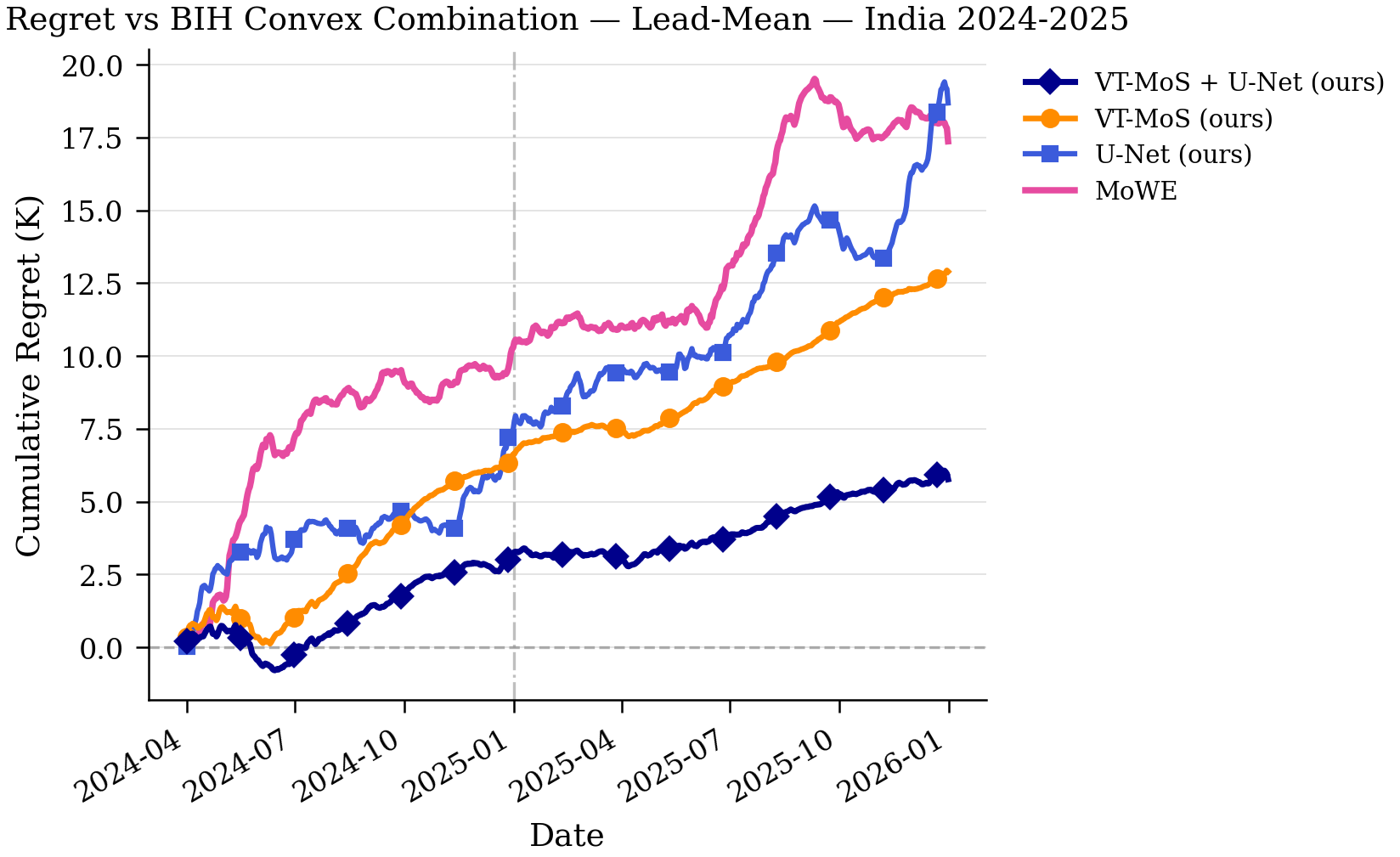}
        \caption{CRPS regret analysis (mixtures only)}
        \label{fig:crpsregret_mix}
    \end{subfigure}
    \hfill
    \begin{subfigure}{0.48\textwidth}
        \centering
        \includegraphics[width=\linewidth]{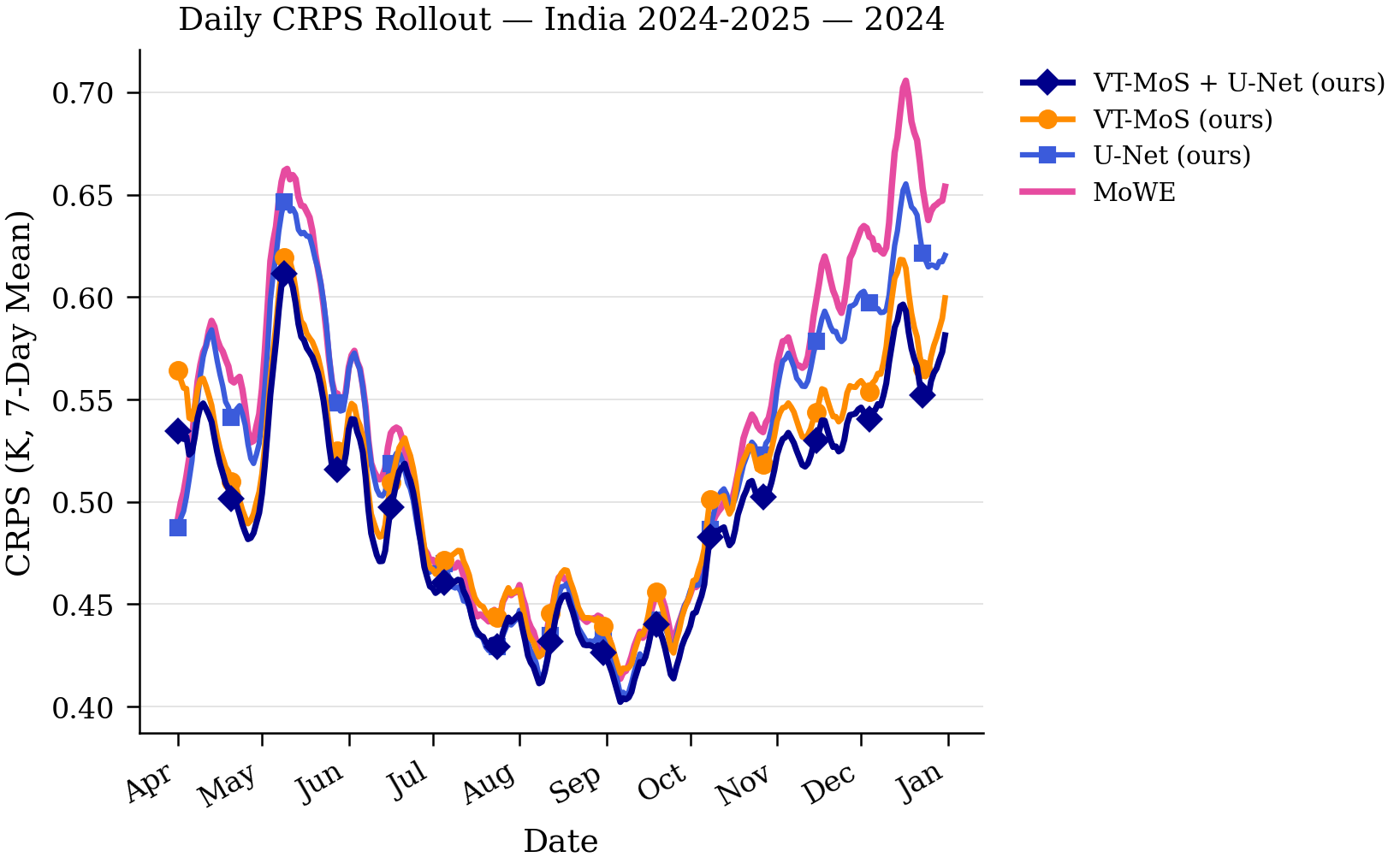}
        \caption{CRPS roll out for 2024}
        \label{fig:crpsrollout_mix}
    \end{subfigure}
    \caption{Evaluation of performance across rollout, lead-time, and regret-based metrics}
    \label{fig:mix_results}
\end{figure}

\begin{table}
  \centering
  \small
  \setlength{\tabcolsep}{4pt}
  \caption{Per-lead CRPS (K, lower is better) on India 2024--2025}
  \label{tab:crps-per-lead}
  \begin{tabular}{lccccccc>{\columncolor{black!4}}c}
  \toprule
  \textbf{Method} & 12h & 24h & 36h & 48h & 60h & 72h & Overall \\
  \midrule
  \textbf{VT-MOS + U-Net}     & \textbf{0.442} & \textbf{0.495} & \textbf{0.495} & \textbf{0.521} & \textbf{0.523} & \textbf{0.543} & \textbf{0.503} \\
  VT-MOS + MoWE            & 0.452 & 0.511 & 0.505 & 0.534 & 0.533 & 0.556 & 0.515 \\
  VT-MOS                       & 0.451 & 0.510 & 0.506 & 0.535 & 0.535 & 0.556 & 0.515 \\
  Vovk-AA                      & 0.465 & 0.524 & 0.534 & 0.554 & 0.566 & 0.578 & 0.537 \\
  U-Net offline                & 0.525 & 0.539 & 0.526 & 0.547 & 0.546 & 0.563 & 0.541 \\
  MoWE               & 0.544 & 0.555 & 0.540 & 0.562 & 0.559 & 0.579 & 0.557 \\
  Equal Weight                 & 0.679 & 0.712 & 0.681 & 0.711 & 0.694 & 0.718 & 0.699 \\
  \midrule
  \multicolumn{8}{l}{\textit{Individual experts}} \\
  FCN3            & 0.474 & 0.558 & 0.575 & 0.614 & 0.627 & 0.652 & 0.583 \\
  FEM             & 0.623 & 0.628 & 0.592 & 0.616 & 0.604 & 0.629 & 0.615 \\
  FGN             & 1.061 & 1.018 & 1.005 & 1.000 & 1.005 & 1.000 & 1.015 \\
  GenCast         & 1.095 & 1.020 & 1.041 & 0.990 & 1.025 & 0.982 & 1.026 \\
  IFS-ENS         & 1.192 & 1.202 & 1.185 & 1.201 & 1.183 & 1.197 & 1.193 \\
  \bottomrule
  \end{tabular}
\end{table}
Figure~\ref{fig:mix_results}c and Figure~\ref{fig:mix_results}d show the regret of different methods with respect to the best mixture in hindsight. The cumulative regret curves grow slowly with $T$, appearing logarithmic, which is consistent with the theoretical guarantee for VT-MOS. The hybrid VT-MOS + U-Net method has the lowest regret throughout the evaluation period, indicating that it stays closest to the hindsight-optimal mixture. In contrast, individual experts accumulate substantially larger regret. This shows that adaptive mixing provides a advantage over both fixed experts and other aggregation baselines.

\paragraph{Ablations:}

We perform some ablations to see which models/model-pairs have the most impact on the mixture's performance, and FCN and FGN are consistently the biggest contributers to the CRPS accuracy of the mixture, as shown in Figure \ref{fig:ablations}.

\section{Conclusion}

We present an extensible offline-online weather forecasting framework while developing novel regret bounds for mixing outputs of different probabilistic models with aggregation algorithms.  Our combination achieves better performance over experts and other methods, with formal guarantees on performance. This is important for a sensitive field like weather forecasting and helps mitigate the black-box nature of AI weather forecasting with lightweight, interpretable expert-based algorithms. Our work has the limitations of a small training period and limited quantity of models used for comparison due to data and compute constraints. Our analysis is also constrained by the intractability of the relevant integrals and the simplex optimization. We plan future work that explores more sources of data for adaptation across stations, satellite, radar, etc, train and evaluate for longer periods of time, more benchmarking, and test-time finetuning frameworks like low-rank adaptation.

\begin{figure}[!htbp]
    \centering
    \includegraphics[width=1.0\linewidth]{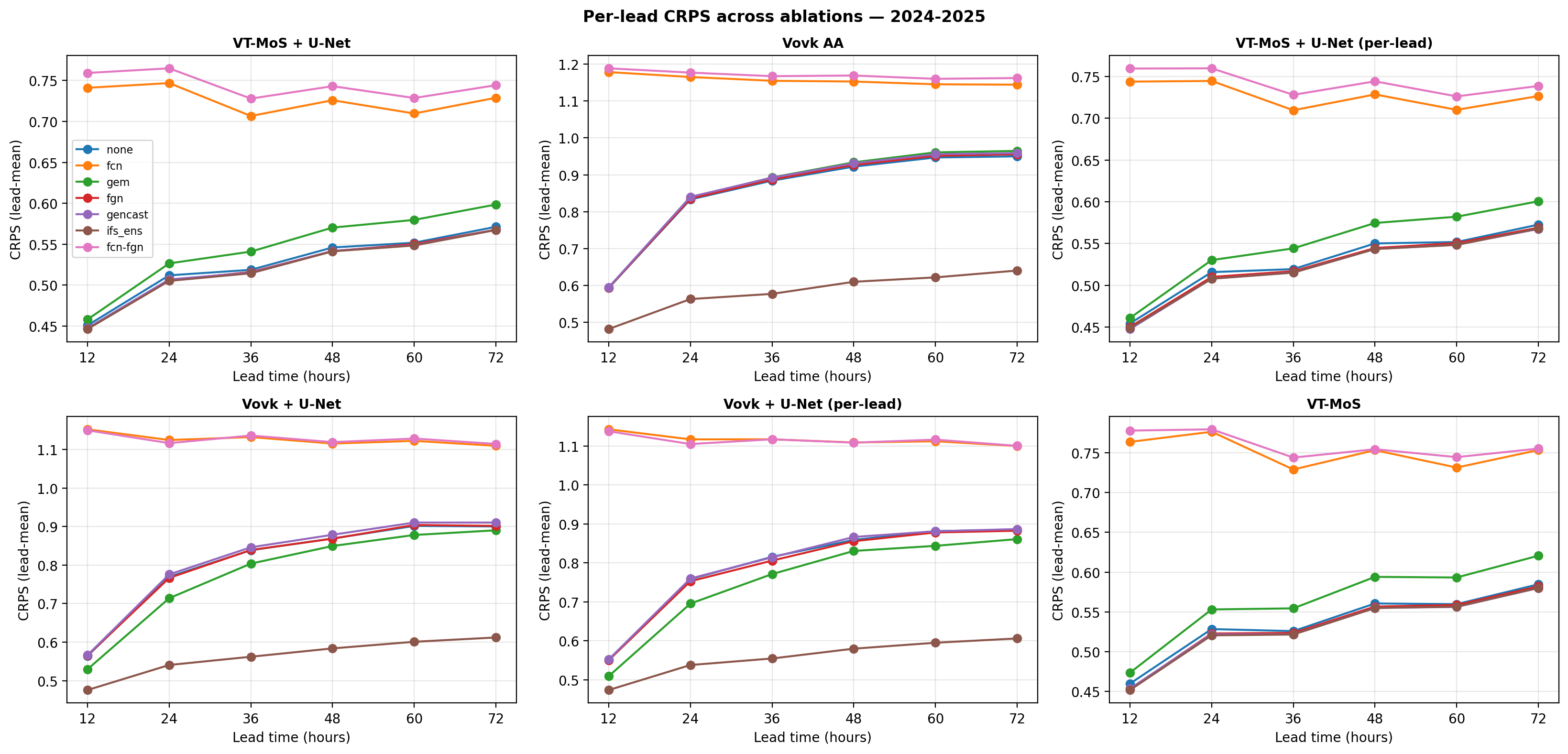}
    \caption{Ablations with labels meaning held out models, showing remarkably better performance of FCN and FGN}
    \label{fig:ablations}
\end{figure}

\clearpage
\newpage
\bibliographystyle{unsrt}  
\bibliography{refs}

\newpage
\appendix

\let\addcontentsline\oldaddcontentsline

\section*{Appendix}

\tableofcontents   

\newpage

\section{Broader Impact}
\label{sec:broader_impact}
Weather forecasting is critical for many fields like agriculture, energy, disaster management, etc, and a more accurate forecast can have broader societal benefits. While potential for positive transformation is there if ML-based weather methods like these are operationalized, over-reliance on these methods should be considered. These risks need to be managed properly and with deep engagement with domain experts, such as the Indian Meteorological Department and other agencies, with more thorough evaluations than done here for real world deployment.

\section{Data}
\label{app:data}

We obtain the FGN, GenCast, IFS-ENS and GEM data from cloud-based datasets like \href{https://console.cloud.google.com/storage/browser/weathernext;tab=objects?prefix=&forceOnObjectsSortingFiltering=false}{WeatherNext}, \href{https://console.cloud.google.com/storage/browser/weatherbench2;tab=objects?prefix=&forceOnObjectsSortingFiltering=false}{WeatherBench2}, \href{https://dynamical.org/}{dynamical.org}, and \href{https://gemv3.salient-open-data.com/assets/README}{Salient GEM v3}, while the FCN3 data is obtained by running the model on our system to obtain a forecast dataset primarily using \href{https://catalog.ngc.nvidia.com/orgs/nvidia/teams/earth-2/models/fourcastnet3?version=0.1.0}{Earth2Studio}. The ERA5 ground truth is available from \href{https://github.com/google-research/arco-era5}{ARCO-ERA5}.

\begin{table}[H]
\centering
\small

\caption{Forecast and reanalysis datasets used for training and evaluation.}
\label{tab:datasets}

The bounding box used for the Indian domain is 6.0 N to 38.0 N and 68.0 E to 98.0 E. 

\renewcommand{\arraystretch}{1.15}

\begin{tabular}{@{}l l c c@{}}
\toprule
\textbf{Name} & \textbf{Developer} & \textbf{Period} &
\textbf{Members} \\
\midrule

\multicolumn{4}{@{}l}{\textit{Forecast experts (training inputs)}} \\
\addlinespace[1pt]

FCN3    & NVIDIA & 2019--2026 & 16 \\
FGN     & Google Research & 2022--2026\textsuperscript & 64 \\
GenCast & Google DeepMind & 2019--2026 & 48 \\
IFS-ENS & ECMWF & 2019--2026 & 50--51 \\
GEM     & ECCC / Meteorological Service of Canada & 2019--2026 & 50 \\

\midrule

\multicolumn{4}{@{}l}{\textit{Ground truth (training target / evaluation)}} \\
\addlinespace[2pt]

ERA5 & ECMWF reanalysis & 2019--2026 & --- \\

\bottomrule
\end{tabular}
\end{table}

We acknowledge that FGN is available only from 2022, so we ignore it from 2019-2021 in training, and have only limited 2022 data for it to contribute to training.

\section{Experiments}
\label{app:expts}

\subsection{Training}
We conduct our experiments on a single H100 GPU, and more training details are given in \ref{app:offline}.

\subsection{Evaluation}

We use the U-Net's predictions as another input to the VT-MOS online algorithm in addition to the other forecasts in an online fashion, and roll out the forecasts over 2 years in an online fashion. Now we specify how we produce our evaluations: the common CRPS estimator (App.~\ref{app:eval-metric}), the
per-lead-time CRPS aggregated over the whole evaluation period
(App.~\ref{app:eval-perlead}), and the two cumulative-regret
quantities --- against the best individual forecast model
(App.~\ref{app:eval-regret-expert}) and against the best static convex
combination in hindsight (App.~\ref{app:eval-regret-bih}).

\paragraph{Scoring protocol and the common CRPS estimator}
\label{app:eval-metric}

\paragraph{Indexing.} The evaluation period is India 2024--2025; the
data split is 2019--2022 (train), 2023 (validation), 2024--2025 (test),
so the window strictly post-dates every data-driven forecaster's
training cut-off. Forecast issue (init) times are indexed $t \in [T]$,
one per available init date in the test span, and lead times by
$\ell \in \{1,\dots,L\}$ with $L = 6$ corresponding to
$\tau \in \{12,24,36,48,60,72\}$\,h. The target is 2-metre temperature
on the $129 \times 121$ India grid $\mathcal{G}$
($[6^\circ,38^\circ]$\,N $\times$ $[68^\circ,98^\circ]$\,E at
$0.25^\circ$); ground truth $y_{t,\ell}$ is the matching ERA5
reanalysis field. There are $N=5$ operational ensemble experts (FCN3,
FGN, GenCast, IFS-ENS, GEM); combiners and the offline U-Net
pseudo-expert are scored on exactly the same metric as the raw experts.

\paragraph{One estimator for every method.} To make all methods
comparable, each method's predictive distribution at $(t,\ell)$ is
materialised as a single per-gridpoint ensemble of $M = 50$ members:
raw experts use their native members (resampled to $M$), and a convex
combiner with weight vector $w \in \Theta$ draws members per pixel in
proportion to $w$ from the expert pools. The reporting CRPS is then the
\emph{unbiased} (fair) kernel estimator, identical to
\texttt{xskillscore.crps\_ensemble} and to the closed form in
Eq.~\eqref{eq:crps}:
\begin{equation}
  \widehat{\CRPS}\big(\{e_i\}_{i=1}^{M},\, y\big)
  \;=\;
  \frac{1}{M}\sum_{i=1}^{M} \lvert e_i - y \rvert
  \;-\;
  \frac{1}{2\,M(M-1)}\sum_{i=1}^{M}\sum_{j=1}^{M}
        \lvert e_i - e_j \rvert ,
  \label{eq:eval-kernel-crps}
\end{equation}
evaluated at every grid point. The $M(M-1)$ denominator (rather than
$M^2$) removes the finite-ensemble bias so that combiners are not
rewarded merely for carrying more effective members than a single
expert. This estimator is the single source of truth for all reported
CRPS; the differentiable closed-form Gaussian-mixture CRPS is used only
as a training loss and never for reporting.

\paragraph{Spatial reduction} Let $\mathcal{G}_{t,\ell}\subseteq\mathcal{G}$ be the
set of grid points with a finite truth value and finite forecast at
$(t,\ell)$. The per-sample score is the masked spatial mean
\begin{equation}
  S^{(\mathrm{m})}_{t,\ell}
  \;=\;
  \frac{1}{\lvert\mathcal{G}_{t,\ell}\rvert}
  \sum_{g\,\in\,\mathcal{G}_{t,\ell}}
  \widehat{\CRPS}^{(\mathrm{m})}_{t,\ell,g},
  \label{eq:eval-spatial-mean}
\end{equation}
for method $\mathrm{m}$. Grid points or whole $(t,\ell)$ samples that
are missing for a method (e.g.\ an expert absent on a given init) are
excluded from the average via \texttt{nanmean} rather than imputed with
zeros, so a method is never silently credited for data it did not
produce.

\paragraph{Per-lead-time CRPS over the whole evaluation period}
\label{app:eval-perlead}

The headline per-lead metric is the time average of
$S^{(\mathrm{m})}_{t,\ell}$ over every init in the test period, taken
independently for each lead. With $\mathcal{T}_\ell = \{\,t :
S^{(\mathrm{m})}_{t,\ell}\text{ finite}\,\}$,
\begin{equation}
  \overline{S}^{\,(\mathrm{m})}_{\ell}
  \;=\;
  \frac{1}{\lvert\mathcal{T}_\ell\rvert}
  \sum_{t\,\in\,\mathcal{T}_\ell}
  S^{(\mathrm{m})}_{t,\ell},
  \qquad \ell = 1,\dots,L,
  \label{eq:eval-perlead}
\end{equation}
This yields one CRPS value per lead,
spanning the entire 2024--2025 evaluation period. The single \emph{overall} CRPS reported per method
is the mean of $S^{(\mathrm{m})}_{t,\ell}$ over all finite $(t,\ell)$
entries; equivalently it summarises the same array collapsed over both
axes.

\paragraph{Cumulative regret vs.\ the best individual forecast model}
\label{app:eval-regret-expert}

The first regret diagnostic compares each combiner to a \emph{per-init
best raw expert} oracle which is a comparator strictly stronger than any
single fixed expert, since it is allowed to switch to whichever raw
ensemble is best on each individual init. Let
$S^{(n)}_{t,\ell}$ denote the per-sample score
\eqref{eq:eval-spatial-mean} of raw expert $n$, restricted to experts
with at least one finite score in the period. The oracle's per-init
score and a method's per-init score are the lead means
\begin{equation}
  b^{\mathrm{raw}}_{t}
  \;=\;
  \frac{1}{L}\sum_{\ell=1}^{L}
  \min_{n\in[N]} S^{(n)}_{t,\ell},
  \qquad
  \bar S^{(\mathrm{m})}_{t}
  \;=\;
  \frac{1}{L}\sum_{\ell=1}^{L} S^{(\mathrm{m})}_{t,\ell},
  \label{eq:eval-bestraw}
\end{equation}
and the per-step regret is
$r^{(\mathrm{m})}_t = \bar S^{(\mathrm{m})}_{t} - b^{\mathrm{raw}}_{t}$.
The plotted quantity is the running cumulative regret
$R^{(\mathrm{m})}_T = \sum_{t=1}^{T} r^{(\mathrm{m})}_t$. Because the raw experts are the
vertices of the simplex $\Theta$, regret against this per-init oracle
is an \emph{upper bound} on regret against the best fixed single expert
and a conservative proxy for the best-in-hindsight regret of
Eq.~\eqref{eq:regret}. (The cumulative-regret rollout figure forms the
oracle as the lead-mean of each expert followed by the per-init
minimum; the envelope analysis below takes the per-$(t,\ell)$ minimum
across experts first and then the lead mean.)

\paragraph{Cumulative regret vs.\ the best static convex combination}
\label{app:eval-regret-bih}

The strongest comparator is the best-in-hindsight (BIH) static convex
combination, i.e.\ the $\min_{w^\star\in\Theta}$ term of
Eq.~\eqref{eq:regret} with a single weight vector held fixed across the
whole period. By Lemma~\ref{lem:quadratic} the CRPS of a convex
combination is quadratic in the weights,
\begin{equation}
  L_{t,\ell}(w)
  \;=\;
  A_{t,\ell}^{\!\top} w
  \;-\;
  \tfrac{1}{2}\, w^{\!\top} B_{t,\ell}\, w ,
  \qquad w \in \Theta,
  \label{eq:eval-bih-quad}
\end{equation}
where $A_{t,\ell}\in\R^{N}$ collects the per-expert skill terms and
$B_{t,\ell}\in\R^{N\times N}$ the pairwise spread terms of
Eq.~\eqref{eq:eval-kernel-crps}, averaged over the patch's valid cells.
$B$ is conditionally negative semi-definite, so $-\tfrac12 w^\top B w$
is convex on $\Theta$ and the hindsight problem
\begin{equation}
  w^\star(\ell)
  \;=\;
  \argmin_{w\,\in\,\Theta}\;
  \sum_{t=1}^{T} L_{t,\ell}(w)
  \label{eq:eval-bih-solve}
\end{equation}
is a convex program over the simplex, solved \emph{per lead} by
sequential least-squares quadratic programming (SLSQP, equality
constraint $\1^{\!\top}w = 1$, bounds $w\ge0$), with a defensive sweep
over the $N$ vertices and the uniform point to guard against numerical
edge cases. For samples missing some experts, the corresponding
entries of $w^\star(\ell)$ are zeroed and the vector renormalised over
the available experts (the same masking rule as the training loss), so
the oracle is never credited with absent forecasts.

\paragraph{Patch restriction and the regret curve.}
Solving \eqref{eq:eval-bih-solve} at full $129\times121$ resolution is
intractable, so the BIH diagnostic restricts to a $10\times8$ patch
(used as the well-defined static-convex comparator). Substituting $w^\star(\ell)$ back into
\eqref{eq:eval-bih-quad} gives the per-$(t,\ell)$ BIH score, and the
reported quantity is the cumulative regret of each method against it,
$\sum_{t}\big(S^{(\mathrm{m})}_{t} - S^{\mathrm{BIH}}_{t}\big)$ as a
lead mean, overlaid on the theoretical $c_{\mathrm{theory}}\ln t + C$
envelope. Because $\Theta$ contains the simplex vertices, this BIH
comparator is at least as strong as the best individual expert of
App.~\ref{app:eval-regret-expert}, so regret staying inside the
logarithmic envelope here is the most stringent empirical confirmation
of Theorem~\ref{thm:reg}.

\section{Empirical CRPS estimator}
\label{app:crps}

For a finite ensemble $\{x_1,\dots,x_M\} \subset \R$ sampled from a
predictive distribution $F$, we use the unbiased $M(M{-}1)$-normalised
plug-in estimator
\begin{equation}
  \widehat{\mathrm{CRPS}}(\{x_i\}, y)
  \;=\;
  \underbrace{\frac{1}{M}\sum_{i=1}^{M} |x_i - y|}_{\text{skill term}}
  \;-\;
  \underbrace{\frac{1}{M(M-1)}
              \sum_{1 \le i < j \le M} \big(x_{(j)} - x_{(i)}\big)}_{\text{spread term}},
  \label{eq:crps_unbiased}
\end{equation}
where $x_{(1)} \le x_{(2)} \le \cdots \le x_{(M)}$ denote the order
statistics of the ensemble. The closed-form
\[
  \CRPS(F, y)
  \;=\; \tfrac{1}{M}\sum_i |e_i - y|
       \;-\; \tfrac{1}{2M^2}\sum_{i,j}|e_i - e_j|
\]
referenced in the main text follows from arguments analogous to those
used in the proof of Lemma~\ref{lem:quadratic} and is omitted.

\section{Offline U-Net: architecture, training, and ablation}
\label{app:offline}

\paragraph{Architecture and parameter count.}
The offline aggregator is a U-Net $f_\theta$ with FiLM lead-time
conditioning \citep{filmcitation}. The encoder applies four
\texttt{DoubleConv} stages of channels $(32, 64, 128, 256)$ separated
by $2 \times 2$ max-pool downsamples, and the decoder mirrors it with
transposed-conv upsamples and U-Net skip connections from the
encoder. Two $1 \times 1$ conv heads on the final feature map produce
expert logits and a scalar bias map. A softmax over experts gives
per-pixel convex weights
\[
  w_{i,n,h,w}
  \;=\; \frac{\exp(z_{i,n,h,w})\, m_{i,n}}
             {\sum_{n'} \exp(z_{i,n',h,w})\, m_{i,n'}},
  \qquad \sum_{n} w_{i,n,h,w} = 1,
\]
where $m_{i,n} \in \{0,1\}$ is a per-init expert-availability mask
that excludes experts with missing data on a given init time. 

\paragraph{Training}

Each expert's per-pixel mean
$\mu_n \in \R^{H \times W}$ and standard deviation
$\sigma_n \in \R^{H \times W}$ are stacked into the input
$S_t \in \R^{N \times 2 \times H \times W}$. The U-Net takes
$(S_t, \ell)$ and produces per-pixel weights
$w \in \R^{N \times H \times W}$, which induce a combined predictive
CDF $\hat F_t^{(w)}$. We train end-to-end against
$$ L_\theta(S_t, \ell) = \CRPS\!\big(\hat F_t^{(w)},\, y_\ell\big),$$
where $y_\ell \in \R^{H \times W}$ is the ground truth at lead time
$\ell$.

We train $\theta$ to minimise the fair (kernel) CRPS of the combined
forecast against the ERA5 ground truth $y \in \R^{H \times W}$. For
every pixel we draw $S = 50$ samples by per-pixel weighted resampling
under $w$ and minimise the unbiased kernel CRPS estimator
\[
  \widehat{\CRPS}(\{\bar y_s\}, y)
  \;=\; \frac{1}{S}\sum_{s=1}^{S} |\bar y_s - y|
       \;-\; \frac{1}{2 S(S-1)}
             \sum_{s \neq s'} |\bar y_s - \bar y_{s'}|,
\]
averaged over pixels with valid ground truth.

\paragraph{Optimisation and runtime.}
We optimise using AdamW (learning rate $3 \times 10^{-4}$, weight
decay $10^{-4}$, batch size $8$) for $30$ epochs, with a
cosine-annealing schedule down to $\eta_{\min} = 10^{-6}$. Each
training example is an (init, lead) pair, so the model sees all
$L = 6$ leads independently and the FiLM head learns lead-specific
behaviour. Years $2019$--$2022$ are used for training, $2023$ for
validation, and $2024$--$2025$ are held out for testing. End-to-end
training takes $\approx 5$\,h on a single H100. All inputs and ground-truth
tensors are stored as Zarr at 0.25° resolution and accessed via xarray.

\paragraph{Per-lead variant.}
We also train an unconditional variant $f_{\theta_\ell}$ \emph{per
lead}: $L$ independent U-Nets with the same encoder--decoder topology
but no lead embedding and no FiLM conditioning, each fit only to its
own lead. This trades parameter sharing across leads for a tighter
fit at each individual horizon and serves as an ablation against the
FiLM-conditioned single model.

\section{Per-pixel CRPS improvement and best raw expert maps}
\label{app:spatial-maps}
 
\begin{figure}[h]
  \centering
  \begin{subfigure}[t]{0.49\linewidth}
    \centering
    \includegraphics[width=\linewidth]{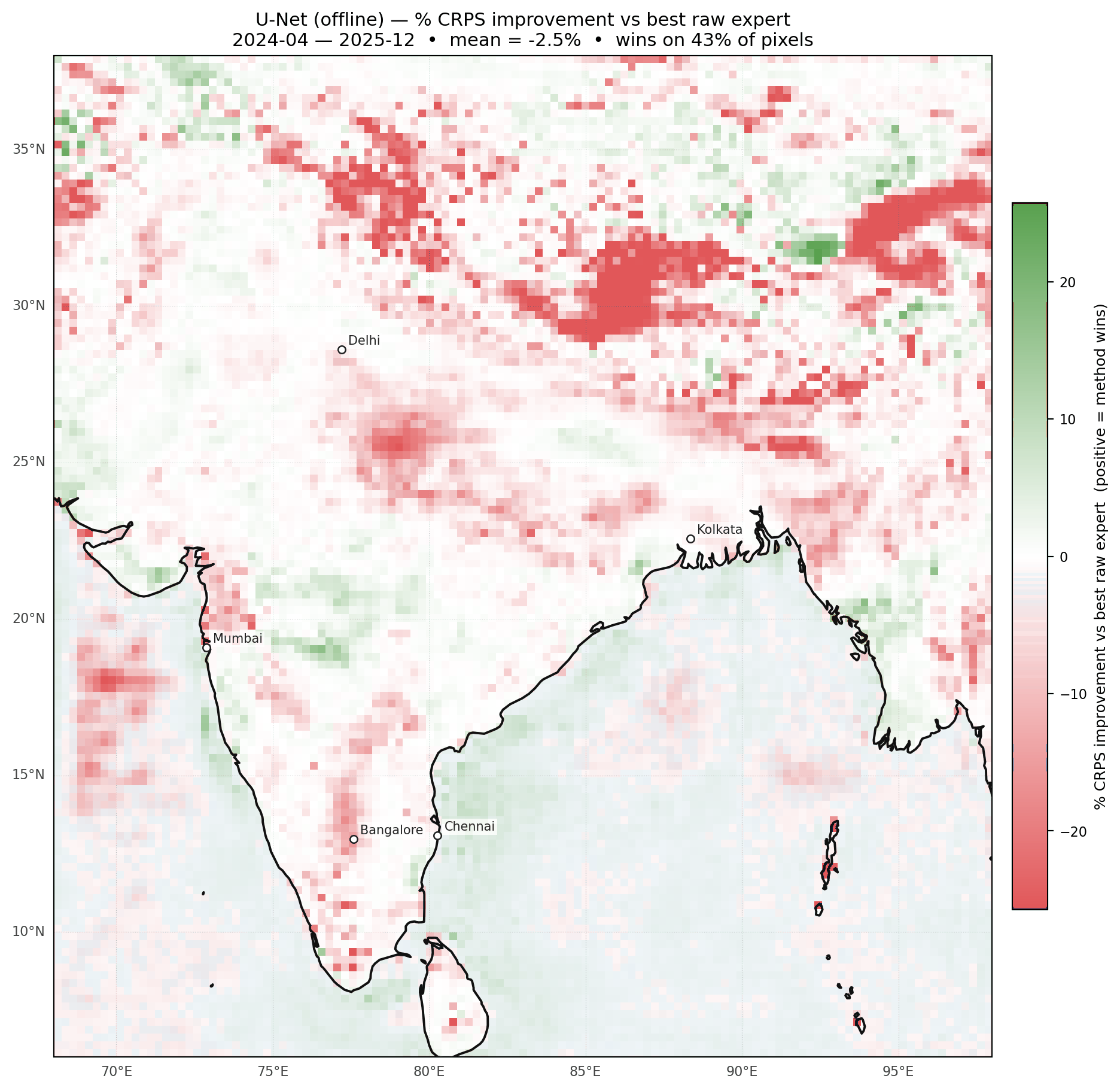}
    \caption{U-Net (offline).}
  \end{subfigure}\hfill
  \begin{subfigure}[t]{0.49\linewidth}
    \centering
    \includegraphics[width=\linewidth]{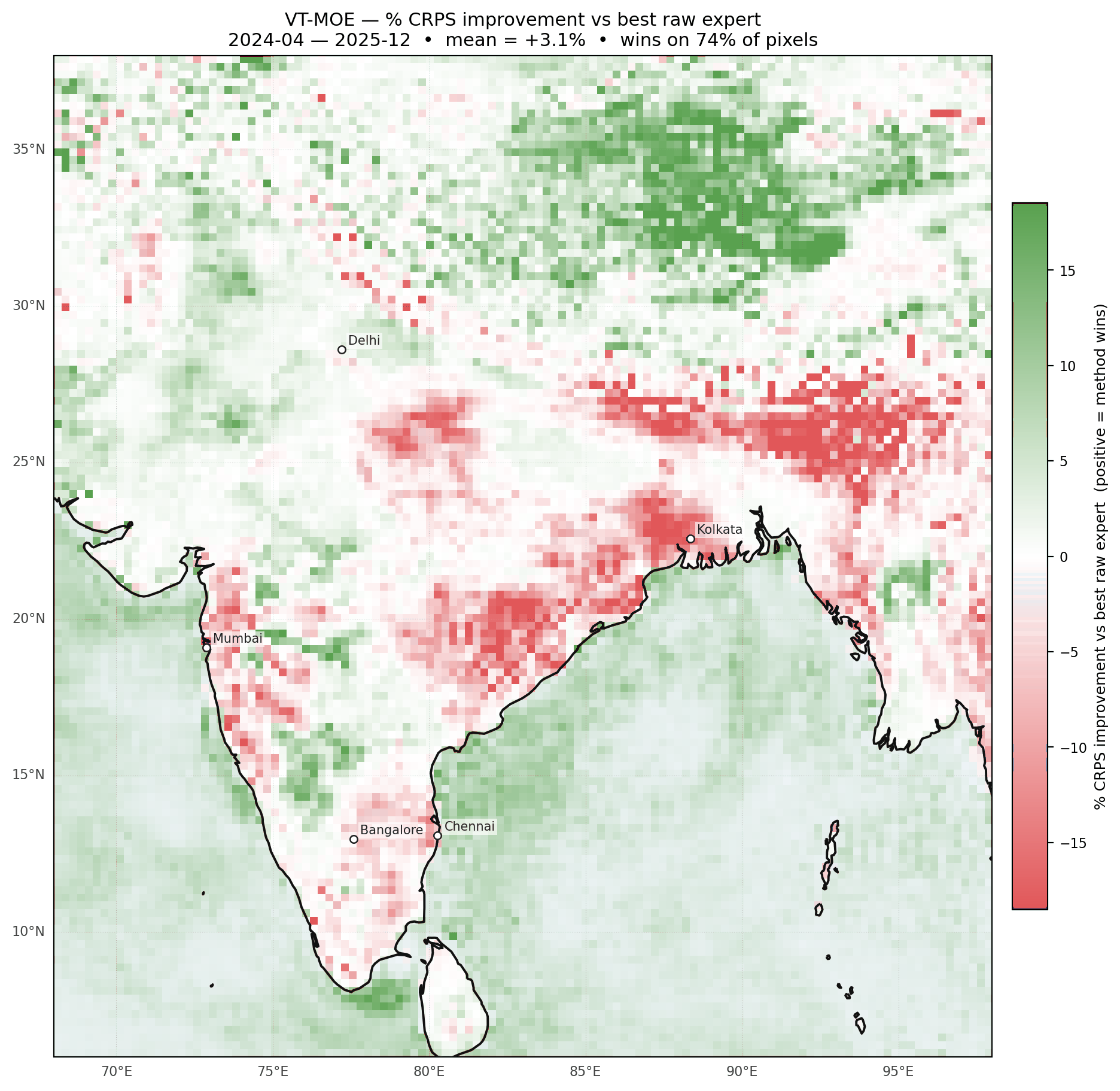}
    \caption{VT-MOS.}
  \end{subfigure}\\[1mm]
  
  \begin{subfigure}[t]{0.49\linewidth}
    \centering
    \includegraphics[width=\linewidth]{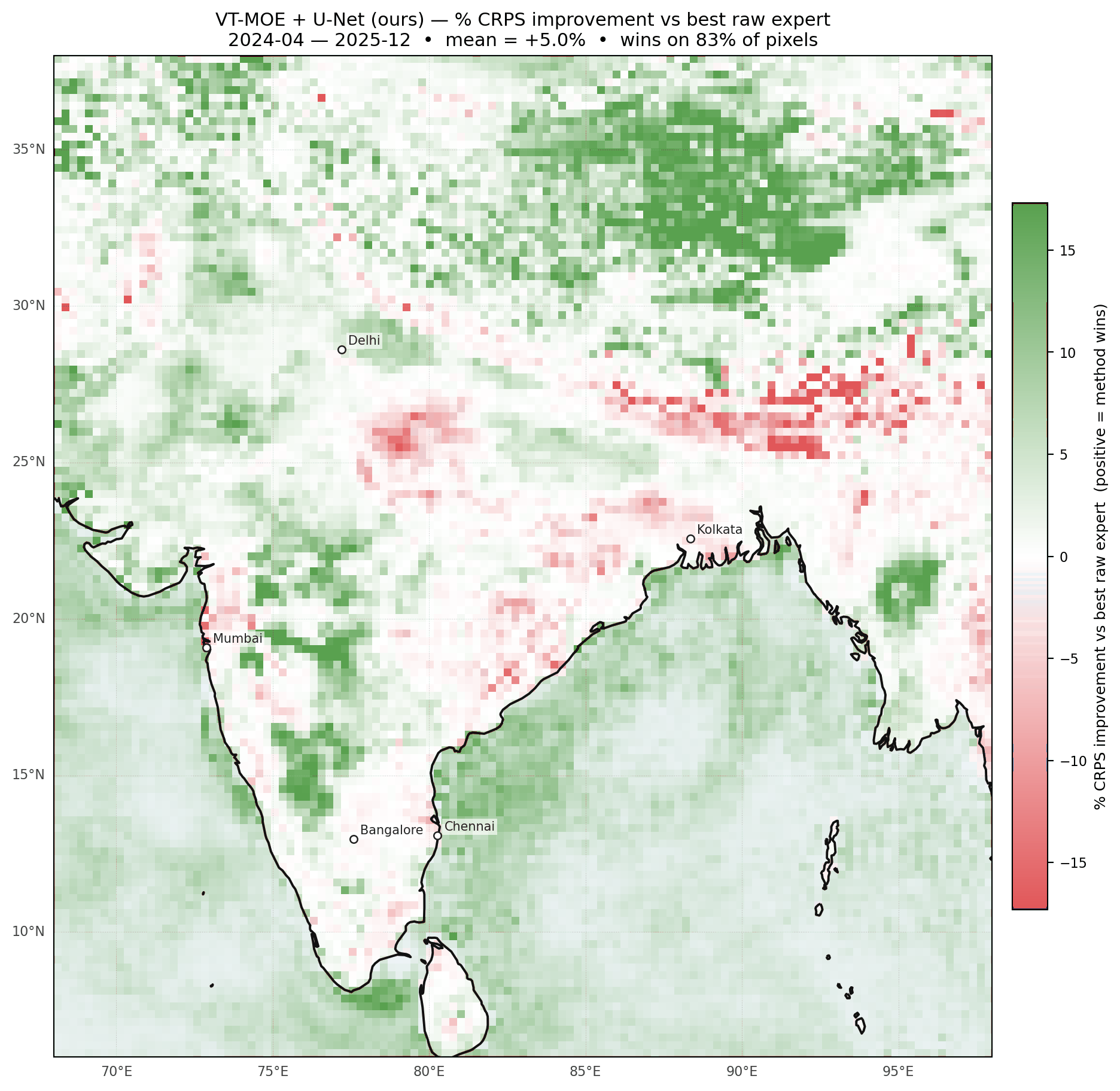}
    \caption{VT-MOS\,+\,U-Net.}
  \end{subfigure}
  \caption{Per-pixel \% CRPS improvement over the best raw expert
  (green i.e. positive = combiner wins), aggregated over 2024--2025. Hybrids
  extend the positive footprint into the trans-Himalayan and
  Western-Ghats belts.}
  \label{fig:avg-improv-maps}
\end{figure}
 
\begin{figure}[h]
  \centering
  \begin{subfigure}[t]{0.32\linewidth}
    \centering
    \includegraphics[width=\linewidth]{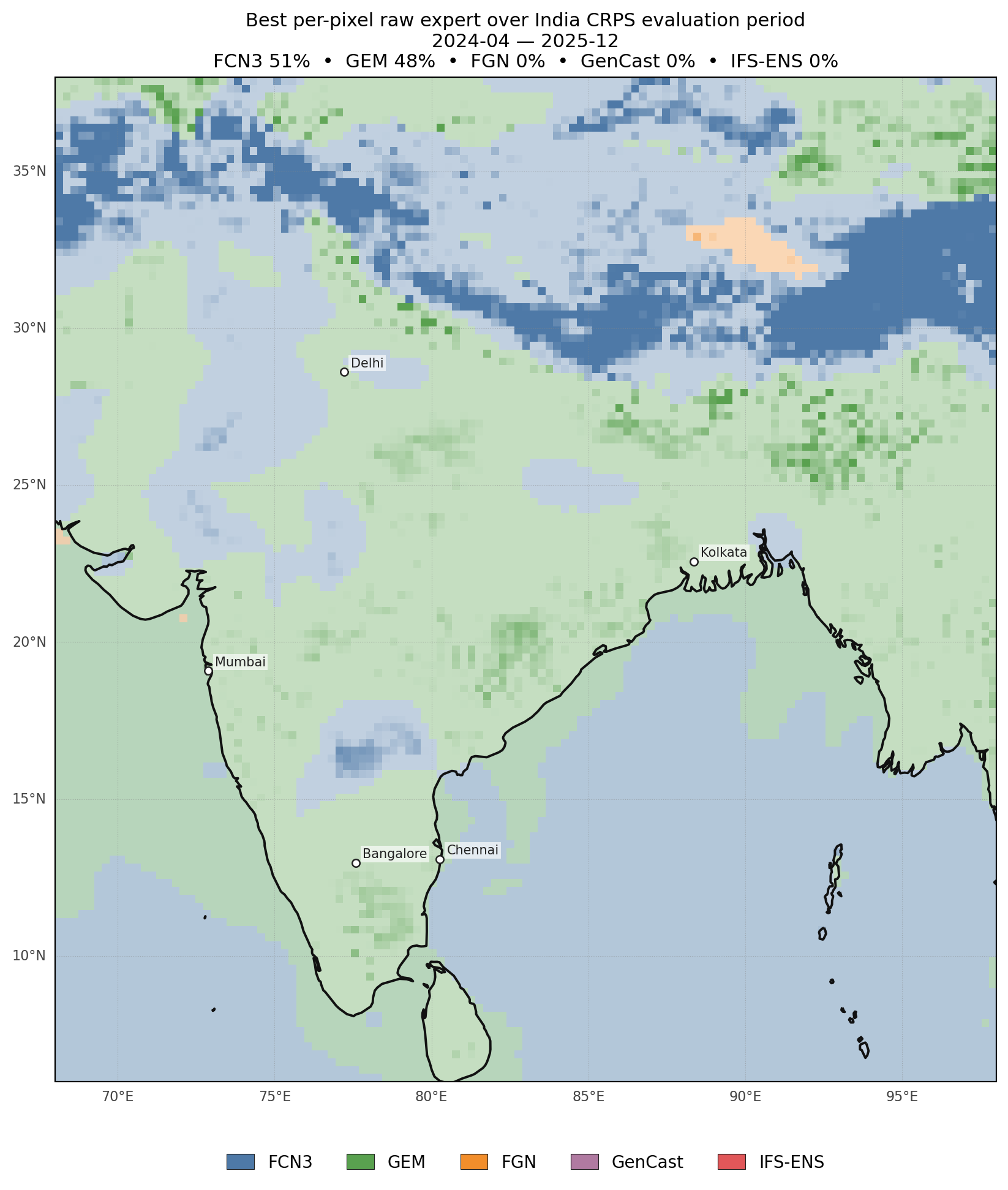}
    \caption{2024--2025.}
  \end{subfigure}\hfill
  \begin{subfigure}[t]{0.32\linewidth}
    \centering
    \includegraphics[width=\linewidth]{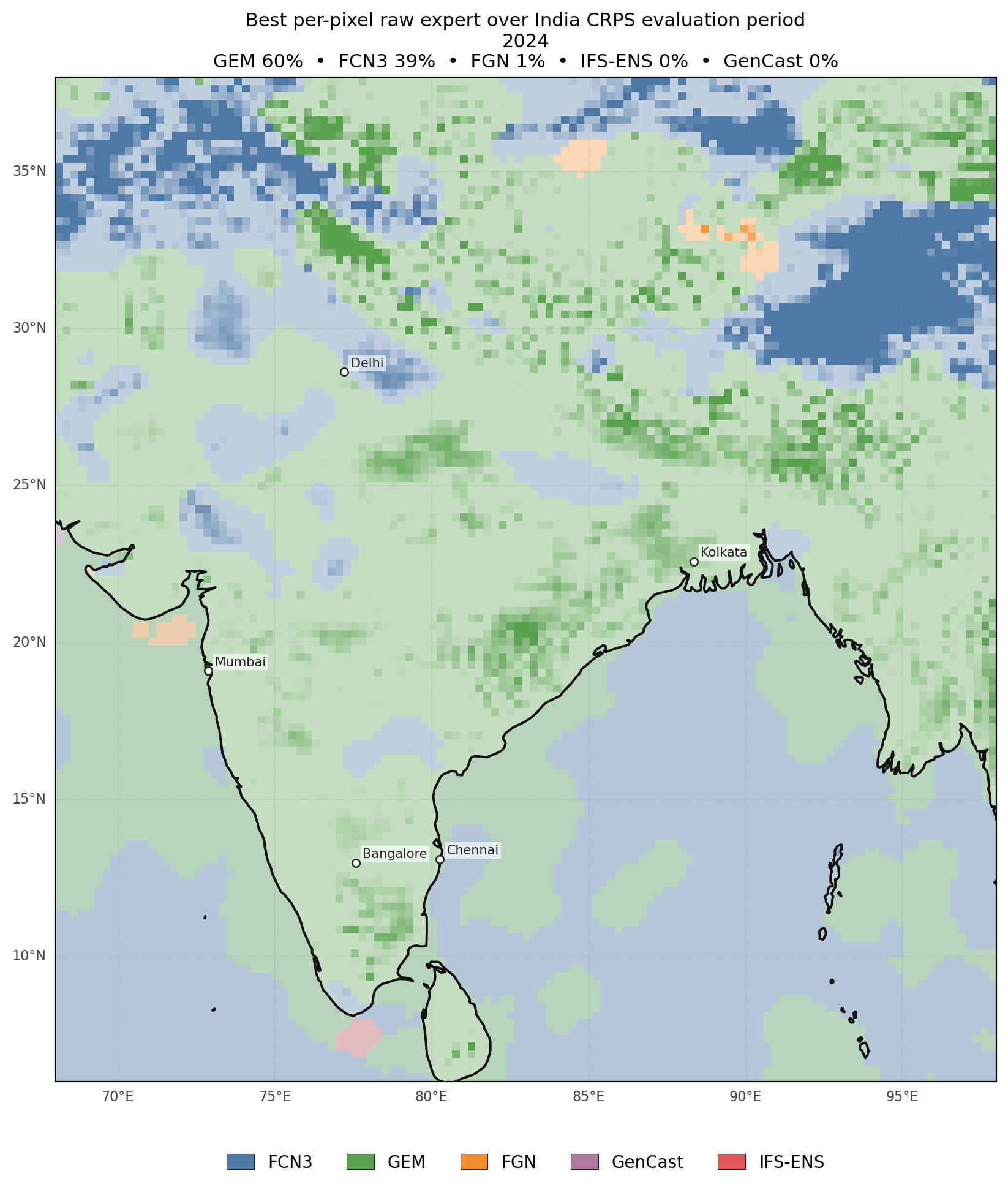}
    \caption{2024.}
  \end{subfigure}\hfill
  \begin{subfigure}[t]{0.32\linewidth}
    \centering
    \includegraphics[width=\linewidth]{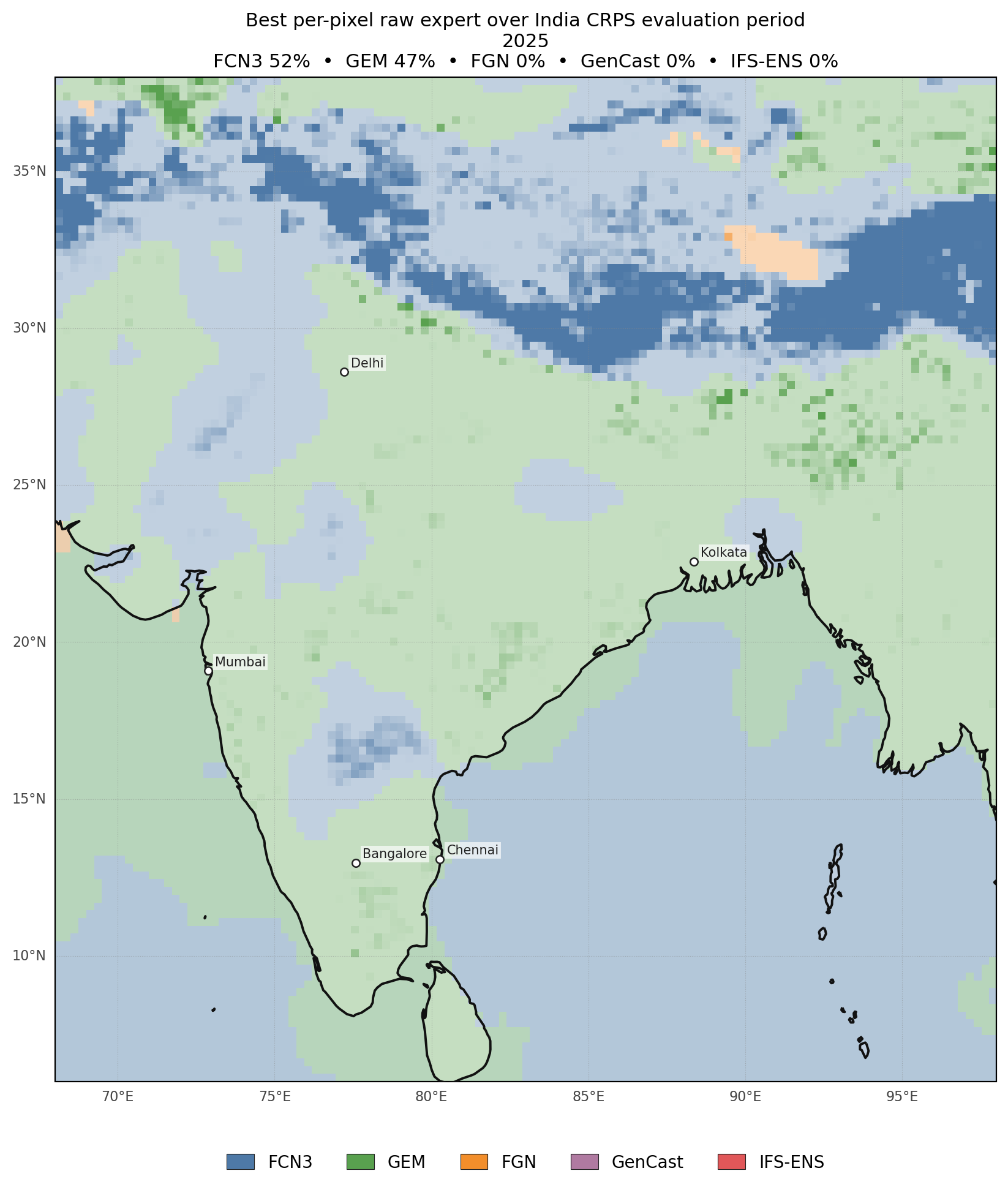}
    \caption{2025.}
  \end{subfigure}
  \caption{Per-pixel best raw expert showing variation\ (categorical, opacity proportional
  to CRPS-margin confidence). The best expert is region- and
  year-dependent.}
  \label{fig:best-raw-expert-maps}
\end{figure}
 
\subsection{BIH-region selection}
\label{app:bih-region}
 
The patch over which we evaluate BIH-regret was selected automatically
as the rectangular sub-region of north India where the proposed hybrid
most exceeds the offline U-Net combiner
(Fig.~\ref{fig:bih-region-diff}). This is the hardest region for the
hybrid: it is where the difference between online adaptation and
static offline weighting is largest, and therefore where stress-testing
the regret bound is most informative.
 
\begin{figure}[h]
  \centering
  \includegraphics[width=0.55\linewidth]{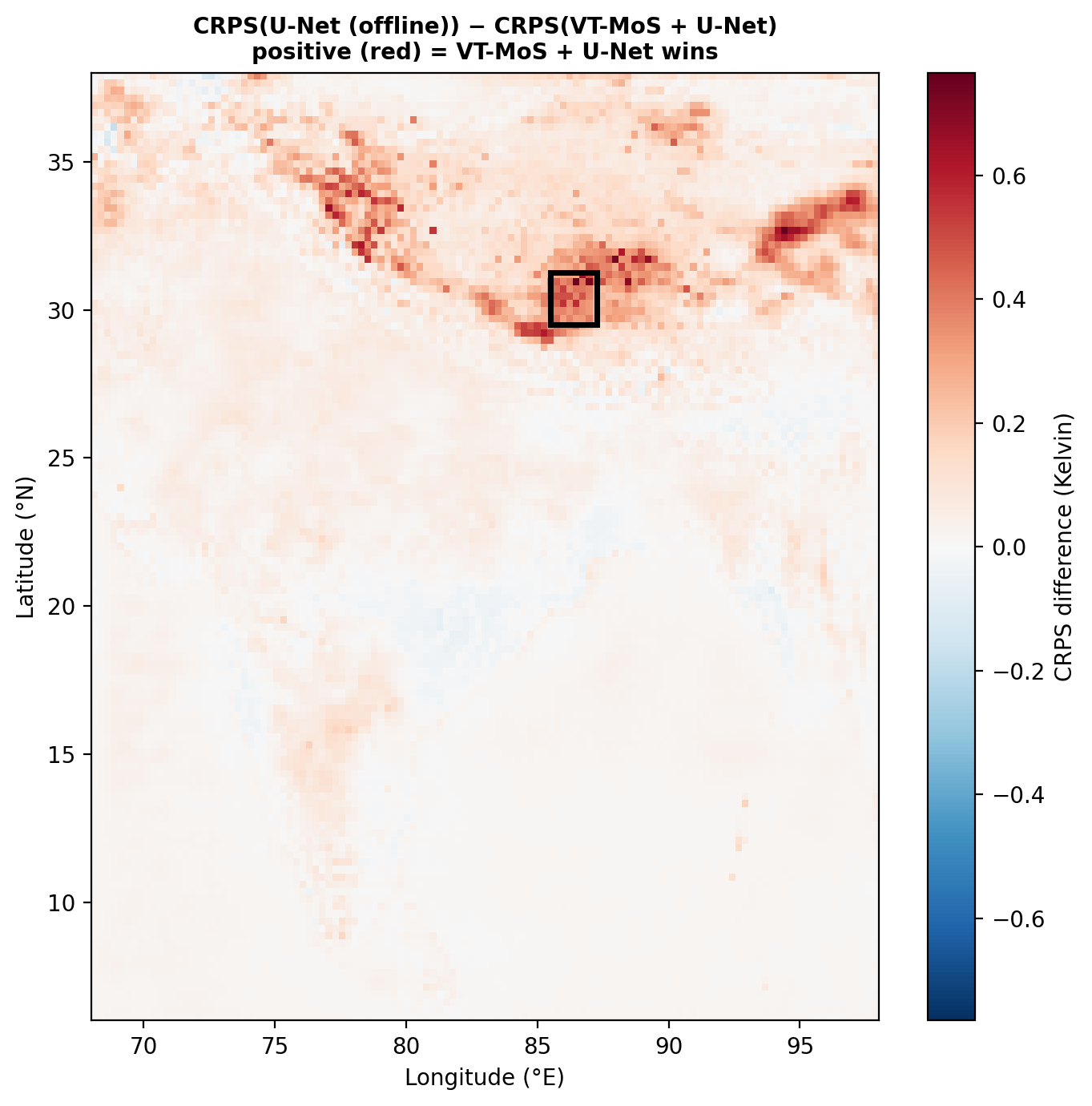}
  \caption{Difference map used to localise the BIH-evaluation patch
  (RED cells = hybrid beats offline U-Net).}
  \label{fig:bih-region-diff}
\end{figure}

\section{Per-city operational evaluation}
\label{app:per-city}

See \ref{tab:per-city}

\begin{table}[h]
  \centering
  \small
  \caption{Patch CRPS (K) over 5$\times$5 city windows, leads averaged,
  2024--2025. Lowest per row bold.}
  \label{tab:per-city}
  \begin{tabular}{lccccccc}
  \toprule
  \textbf{City} & \textbf{VT-MOS+U-Net} & Vovk+U-Net & VT-MOS & U-Net (off) & Vovk AA & Eq.Wt. & Raw FCN3 \\
  \midrule
  Delhi      & \textbf{0.588} & 0.977 & 0.600 & 0.644 & 1.055 & 0.885 & 0.644 \\
  Mumbai     & \textbf{0.357} & 0.663 & 0.378 & 0.359 & 0.746 & 0.564 & 0.414 \\
  Chennai    & \textbf{0.349} & 0.759 & 0.366 & 0.364 & 0.840 & 0.523 & 0.391 \\
  Kolkata    & \textbf{0.470} & 0.873 & 0.484 & 0.482 & 0.939 & 0.700 & 0.535 \\
  Bangalore  & \textbf{0.420} & 0.760 & 0.435 & 0.434 & 0.821 & 0.580 & 0.476 \\
  Hyderabad  & 0.404          & 0.700 & 0.418 & \textbf{0.403} & 0.784 & 0.566 & 0.439 \\
  Ahmedabad  & \textbf{0.435} & 0.811 & 0.452 & 0.458 & 0.890 & 0.642 & 0.499 \\
  Pune       & \textbf{0.378} & 0.706 & 0.395 & 0.387 & 0.788 & 0.542 & 0.434 \\
  Jaipur     & \textbf{0.518} & 0.910 & 0.534 & 0.554 & 0.987 & 0.776 & 0.581 \\
  Lucknow    & \textbf{0.527} & 0.951 & 0.541 & 0.578 & 1.020 & 0.821 & 0.582 \\
  \bottomrule
  \end{tabular}
\end{table}

\section{Weight-contribution dynamics and hybrid stability}
\label{app:weights}
 
The temporal evolution of each combiner's expert weights, decomposed
into three groups (offline U-Nets, online aggregators, hybrids),
appears in Fig.~\ref{fig:weight-groups}. Hybrid stability is shown in
Fig.~\ref{fig:hybrid-stability}: the hybrid collapses onto the U-Net
when the U-Net pseudo-expert band is high (and inherits its variance);
when the band shrinks, the online aggregator over the raw experts takes
over and damps the variance further.
 
\begin{figure}[h]
  \centering
  \begin{subfigure}[t]{0.49\linewidth}
    \centering
    \includegraphics[width=\linewidth]{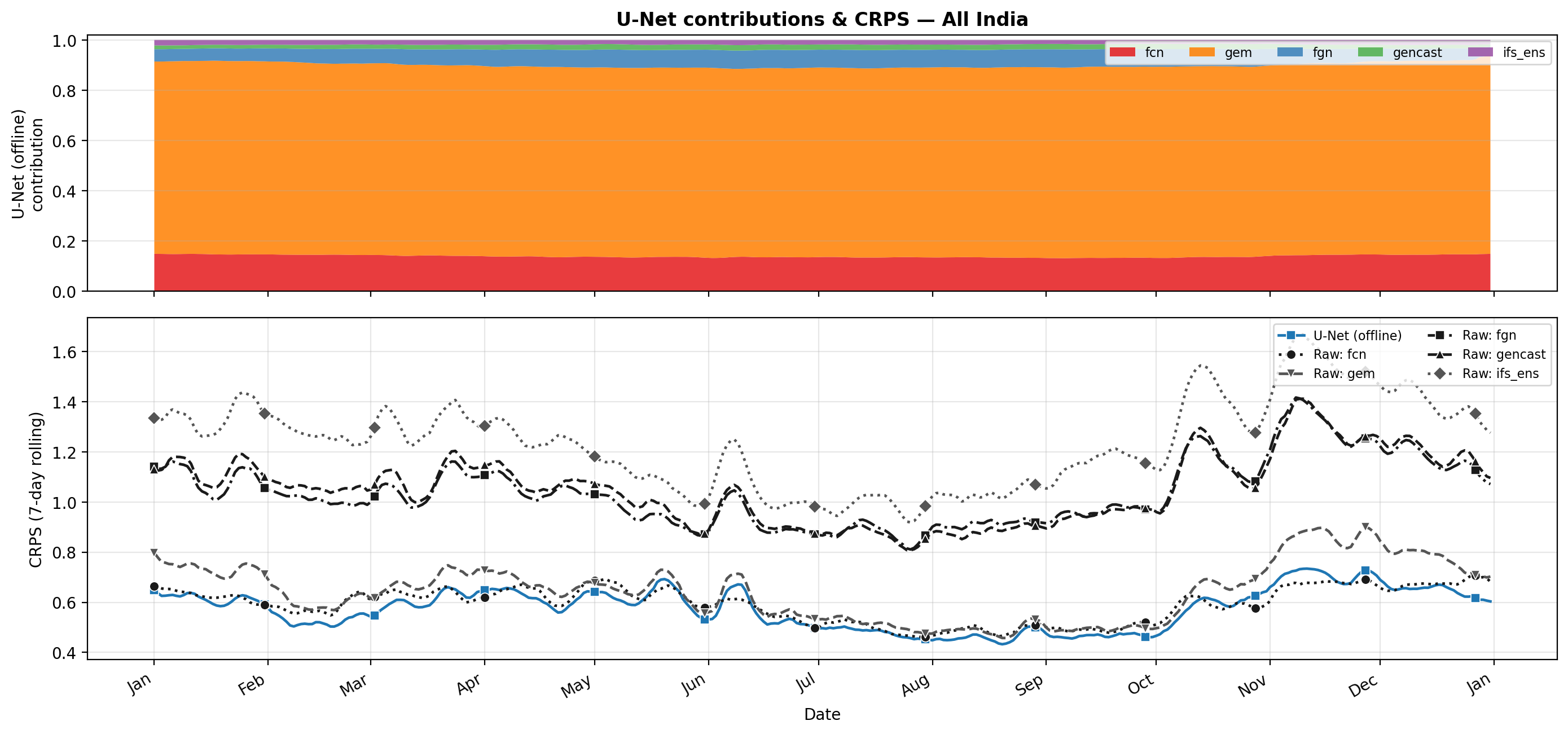}
    \caption{Group 1: offline U-Nets.}
  \end{subfigure}\hfill
  \begin{subfigure}[t]{0.49\linewidth}
    \centering
    \includegraphics[width=\linewidth]{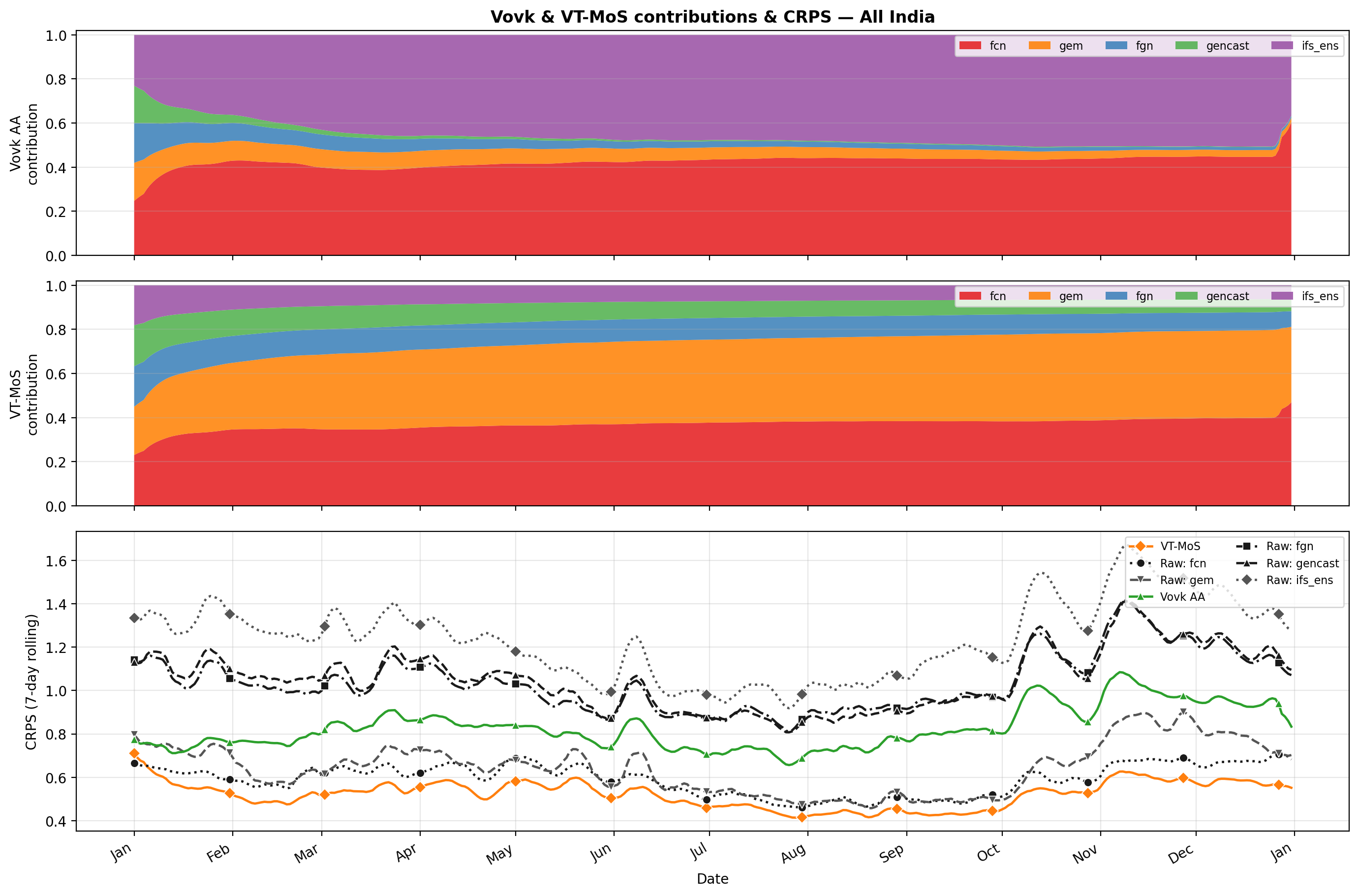}
    \caption{Group 2: Vovk and VT-MOS.}
  \end{subfigure}\\[1mm]
  \begin{subfigure}[t]{0.7\linewidth}
    \centering
    \includegraphics[width=\linewidth]{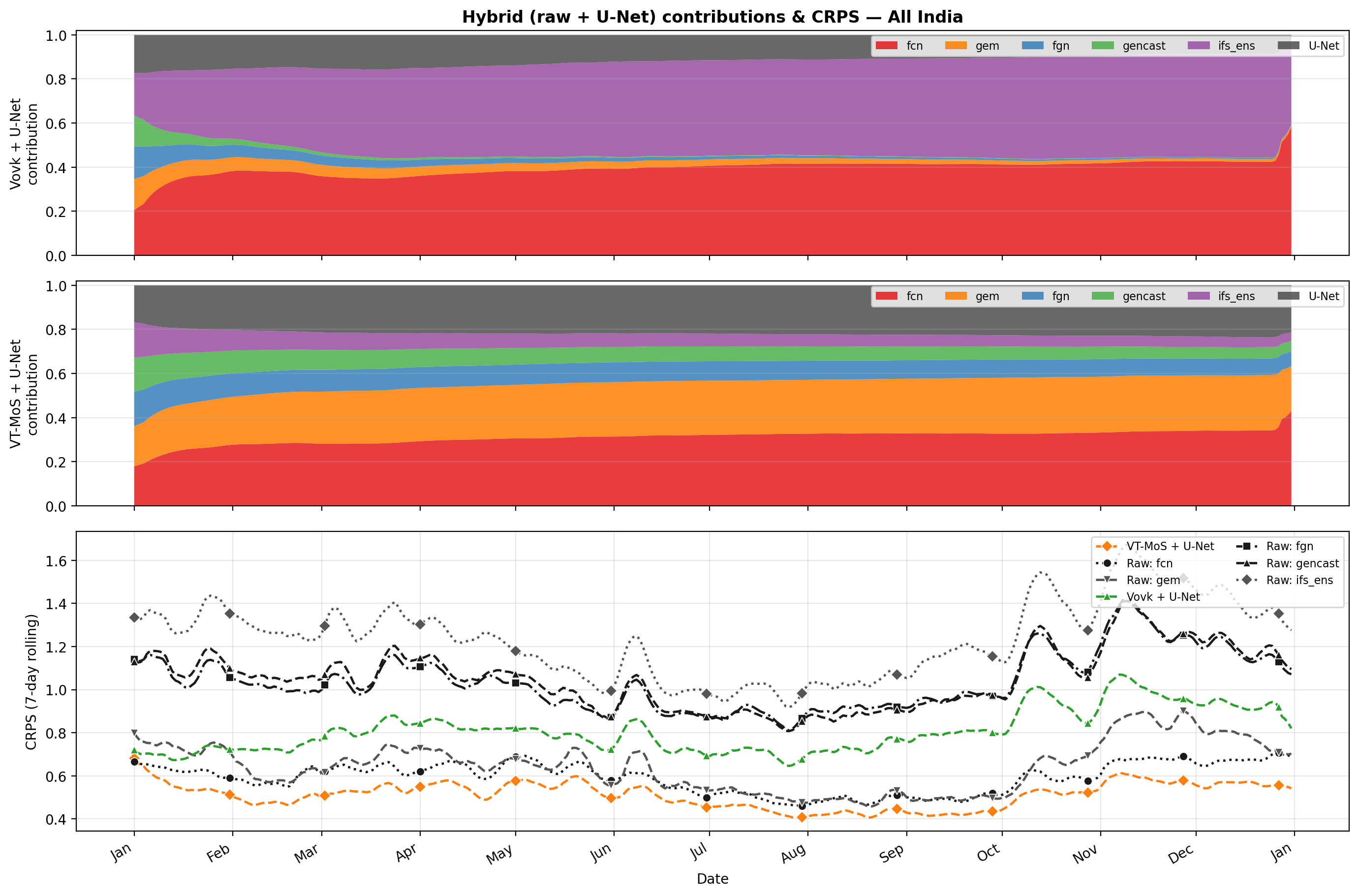}
    \caption{Group 3: Hybrids (Vovk\,+\,U-Net and VT-MOS\,+\,U-Net).}
  \end{subfigure}
  \caption{Three-group decomposition of expert weight contributions
  (top of each panel) alongside the spatial-mean CRPS rollout
  (bottom). The dark-grey band in group 3 is the U-Net pseudo-expert's
  trust coefficient.}
  \label{fig:weight-groups}
\end{figure}
 
\begin{figure}[h]
  \centering
  \includegraphics[width=0.85\linewidth]{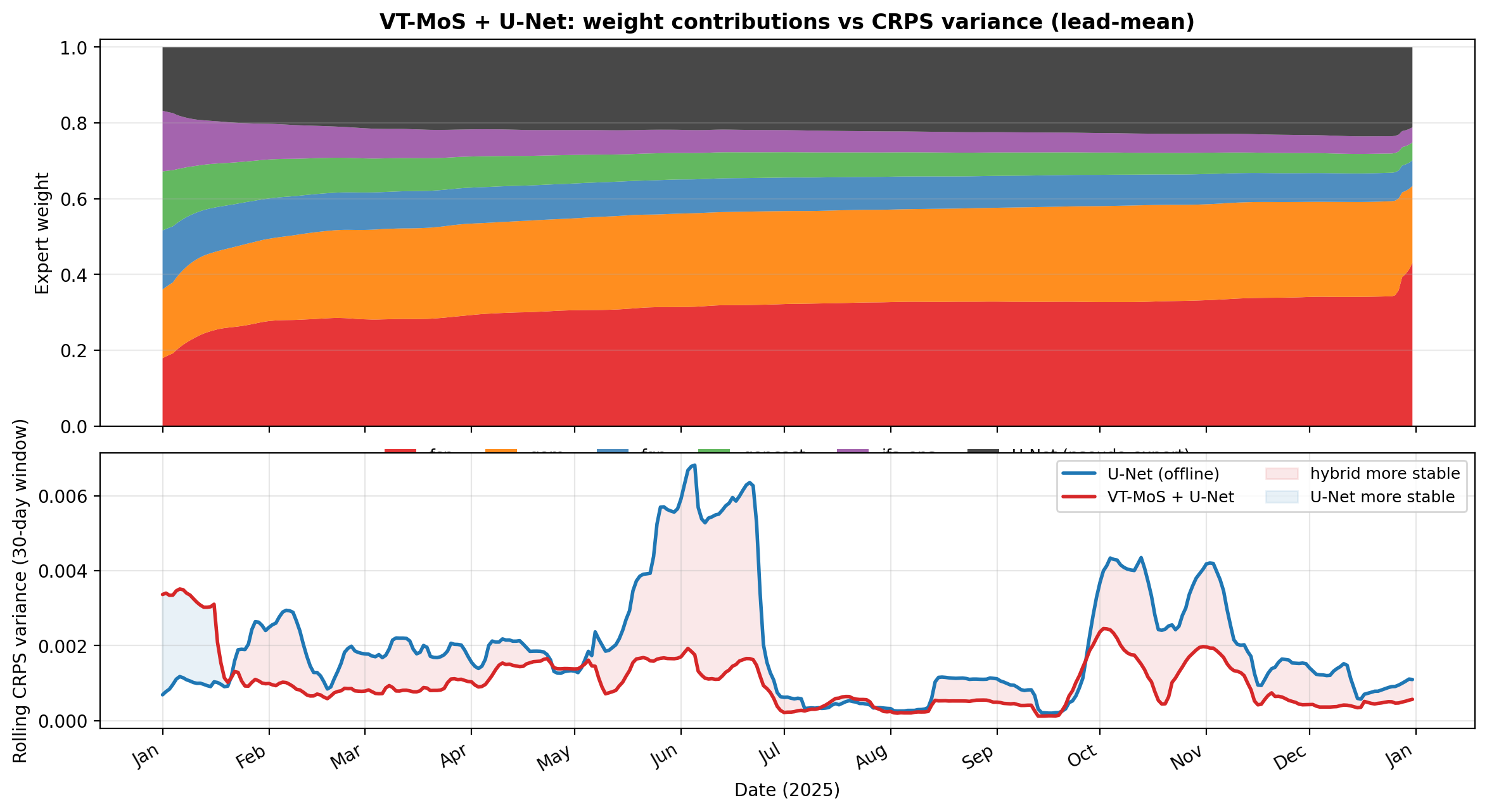}
  \caption{(Top) VT-MOS\,+\,U-Net stacked expert weights over 2025.
  (Bottom) 30-day rolling CRPS variance of U-Net (offline) alone
  (dashed) vs the hybrid (solid).}
  \label{fig:hybrid-stability}
\end{figure}

\section{Tail events: cold/normal/hot}
\label{app:tail}

See Table \ref{app:tail} and Figure \ref{fig:tail-events-bars}
 
\begin{figure}[h]
  \centering
  \includegraphics[width=0.85\linewidth]{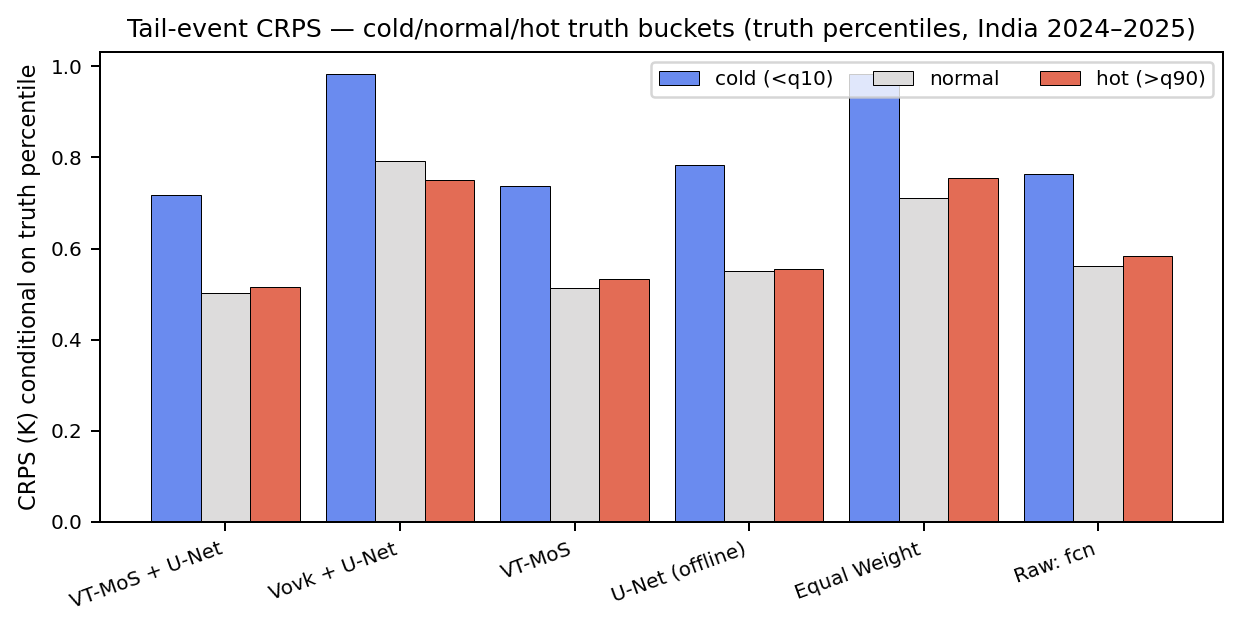}
  \caption{Per-bucket CRPS conditional on truth percentile.
  VT-MOS+U-Net has lower CRPS per bin}
  \label{fig:tail-events-bars}
\end{figure}
 
\begin{table}[h]
  \centering
  \small
  \caption{Tail-event CRPS (K) per truth-percentile bucket.
  Cold = $y<q_{0.1}$, hot = $y>q_{0.9}$}
  \label{tab:tail-events}
  \begin{tabular}{lcccccc}
  \toprule
  \textbf{Method} & overall & cold & normal & hot \\
  \midrule
  \textbf{VT-MOS+U-Net} & \textbf{0.526} & \textbf{0.717} & \textbf{0.503} & \textbf{0.516} \\
  Vovk+U-Net           & 0.806 & 0.983 & 0.791 & 0.751\\
  VT-MOS               & 0.537 & 0.738 & 0.513 & 0.534\\
  U-Net (offline)      & 0.574 & 0.784 & 0.550 & 0.555\\
  Equal Weight         & 0.742 & 0.983 & 0.710 & 0.756\\
  Raw: FCN3            & 0.583 & 0.764 & 0.561 & 0.583\\
  \bottomrule
  \end{tabular}
\end{table}

\section{VT-MOS Monte-Carlo sensitivity}
\label{app:m-sens}
 
Theorem~\ref{thm:reg} assumes $M \to \infty$. Empirically the per-step
predictor RMSE against an $M_{\mathrm{ref}} = 32{,}000$ reference
decays as $O(M^{-1/2})$ as predicted (Fig.~\ref{fig:m-sensitivity},
Tab.~\ref{tab:m-sensitivity}). At our production setting $M = 1000$,
the per-step RMSE is $\approx 0.013$ on $\hat F$ values that lie in
$[0,1]$, comfortably below typical CRPS gradients across the test
window. Increasing to $M = 4000$ reduces this further to $\approx
0.011$, with diminishing returns; we settled on $M = 1000$ as a
production tradeoff with negligible impact on downstream CRPS.
 
\begin{figure}[h]
  \centering
  \includegraphics[width=0.55\linewidth]{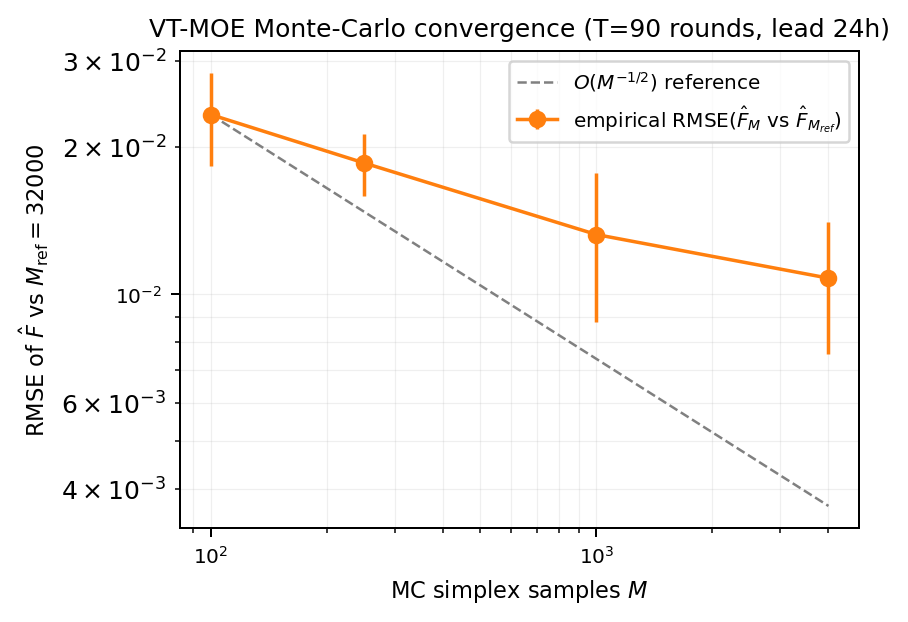}
  \caption{VT-MOS Monte-Carlo convergence: per-step RMSE of the
  predictor $\hat F_M$ against an $M_{\mathrm{ref}}=32{,}000$ reference,
  averaged over 8 random seeds and $T=90$ rounds (lead 24\,h).}
  \label{fig:m-sensitivity}
\end{figure}
 
\begin{table}[h]
  \centering
  \small
  \caption{Per-step VT-MOS predictor RMSE vs $M_{\mathrm{ref}}=32{,}000$
  reference, mean $\pm$ std over 8 seeds.}
  \label{tab:m-sensitivity}
  \begin{tabular}{rcc}
  \toprule
  $M$ & RMSE (mean $\pm$ std) & relative cost \\
  \midrule
  100  & $0.0233 \pm 0.0050$ & $1.0\times$ \\
  250  & $0.0186 \pm 0.0027$ & $2.5\times$ \\
  1000 & $0.0133 \pm 0.0045$ & $10\times$ \\
  4000 & $0.0108 \pm 0.0032$ & $40\times$ \\
  \bottomrule
  \end{tabular}
\end{table}

\section{Station-level evaluation}
\label{app:stations}

See \ref{fig:station-headline}
\begin{figure}[h]
  \centering
  \begin{subfigure}[t]{0.49\linewidth}
    \centering
    \includegraphics[width=\linewidth]{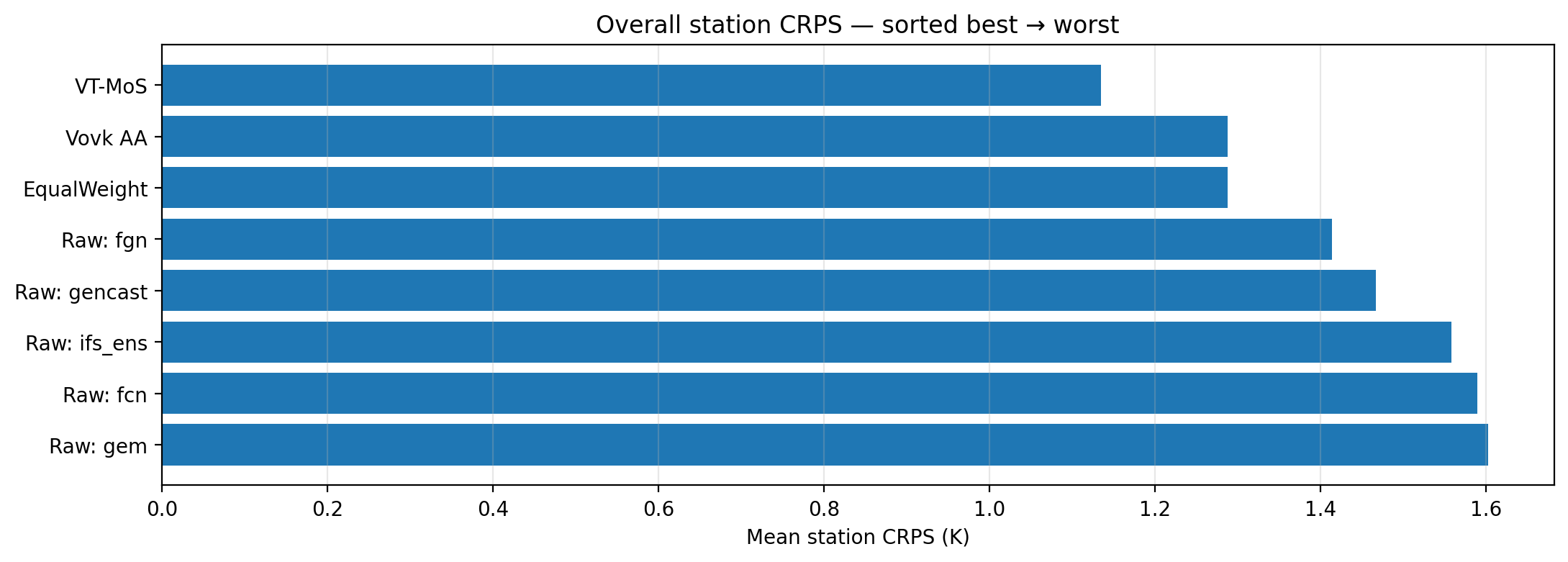}
    \caption{Overall mean CRPS.}
  \end{subfigure}\hfill
  \begin{subfigure}[t]{0.49\linewidth}
    \centering
    \includegraphics[width=\linewidth]{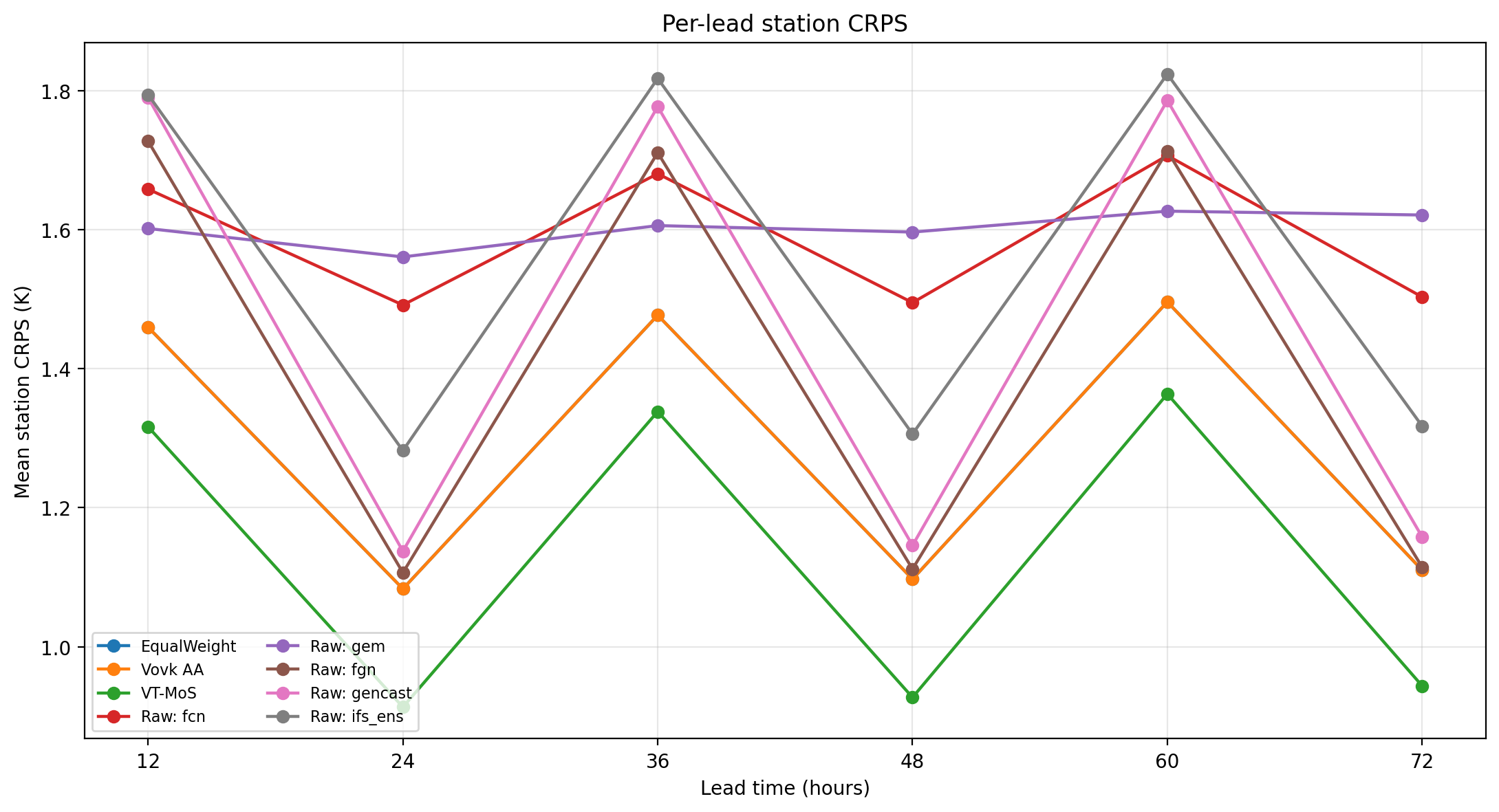}
    \caption{Per-lead mean CRPS.}
  \end{subfigure}\\[1mm]
  \begin{subfigure}[t]{0.49\linewidth}
    \centering
    \includegraphics[width=\linewidth]{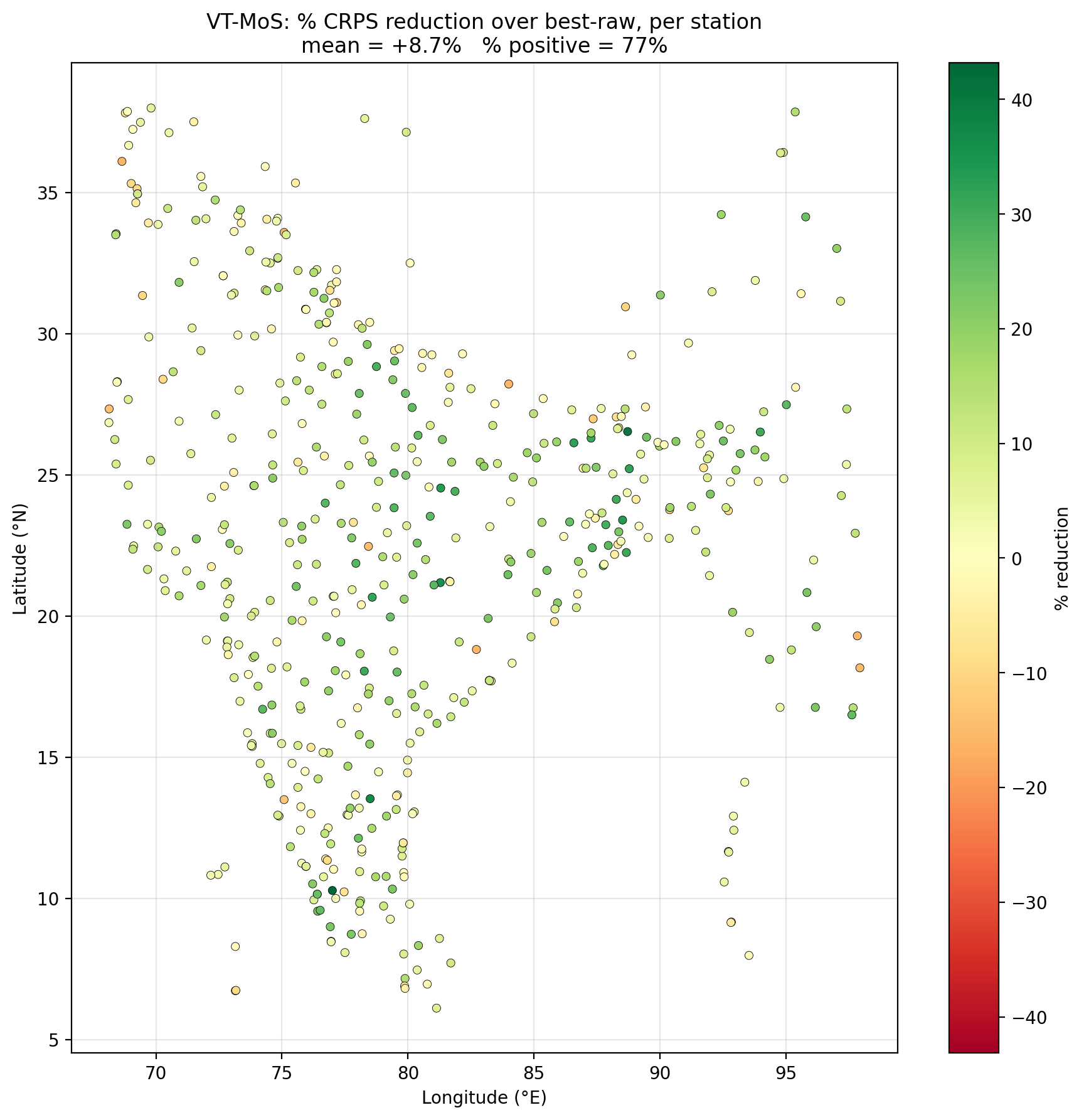}
    \caption{Per-station \% CRPS reduction of VT-MOS.}
  \end{subfigure}\hfill
  \begin{subfigure}[t]{0.49\linewidth}
    \centering
    \includegraphics[width=\linewidth]{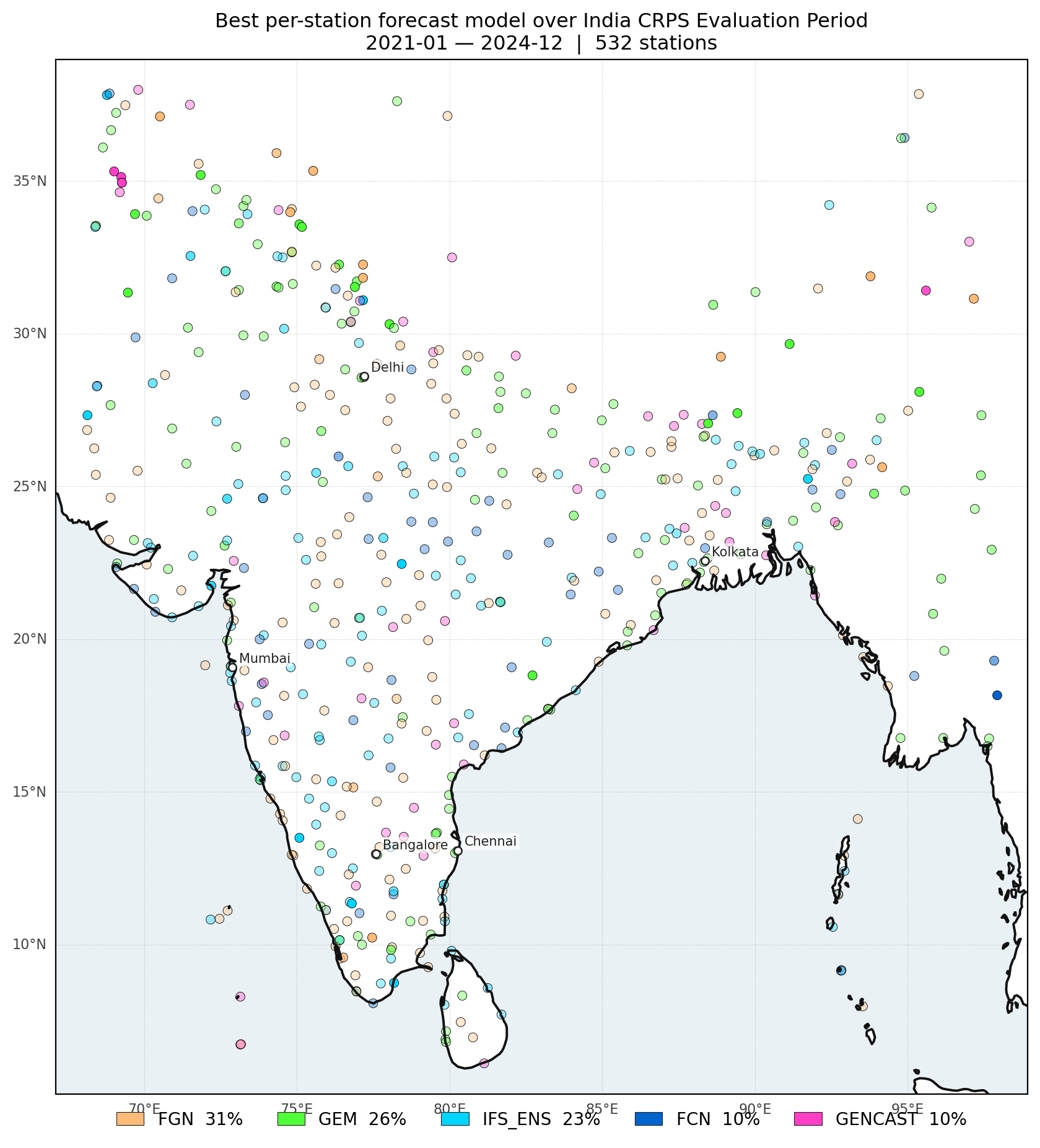}
    \caption{Per-station best raw expert.}
  \end{subfigure}
  \caption{Station-level evaluation across 546 StationBench-India
  sites, 2021--2024. There is more variance between performance of each model on data that neither has seen in training, which further motivates combination of diverse models on sparse station data}
  \label{fig:station-headline}
\end{figure}

\FloatBarrier

\section{Other Regions}
\label{app:otherregions}

\begin{figure}[!htbp]
  \centering
  \includegraphics[width=1\linewidth]{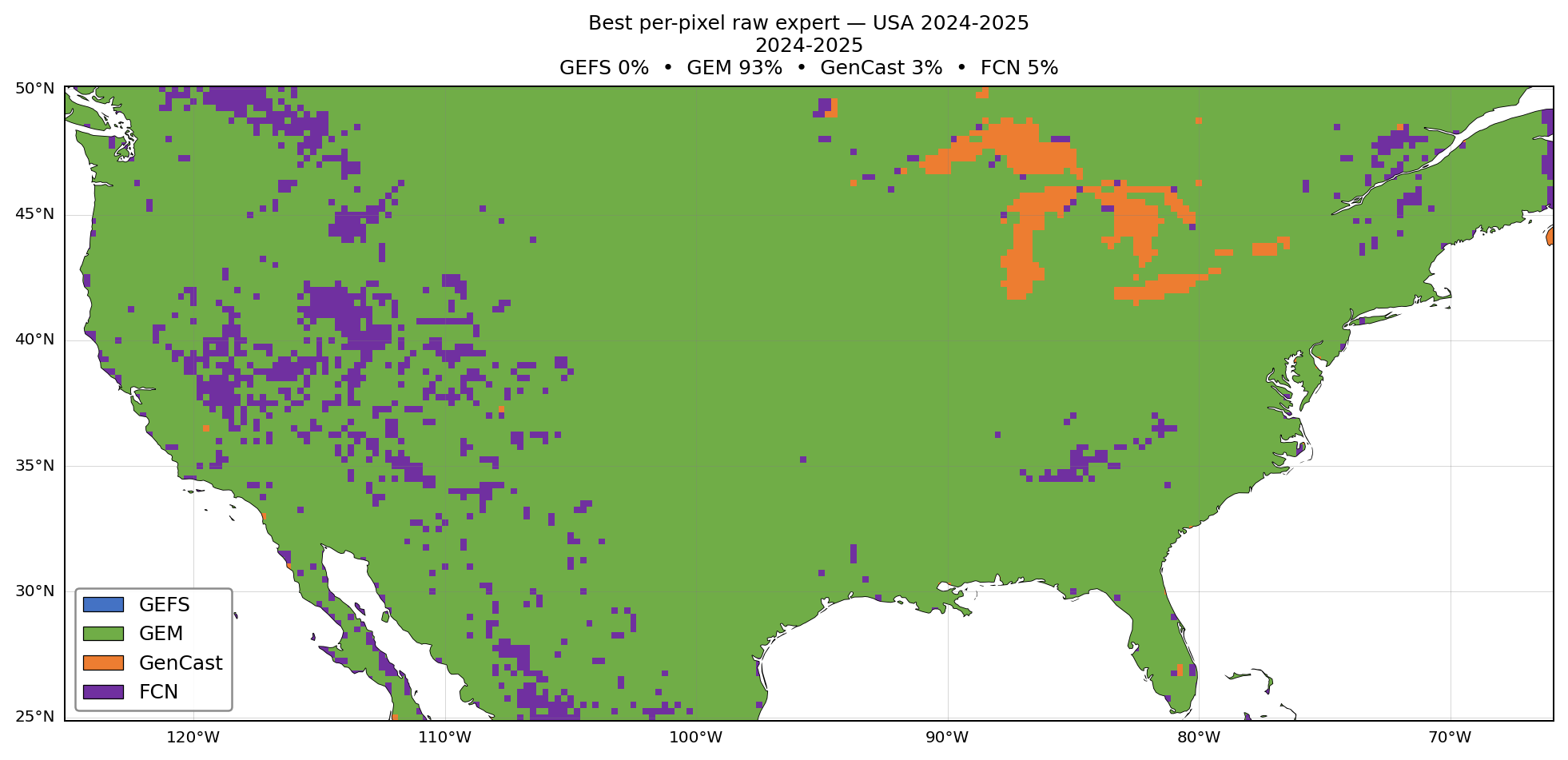}
  \caption{Per-pixel best raw expert showing spatial variation.}
  \label{fig:usa-best-expert-spatial}
\end{figure}

\begin{figure}[!htbp]
  \centering
  \includegraphics[width=1\linewidth]{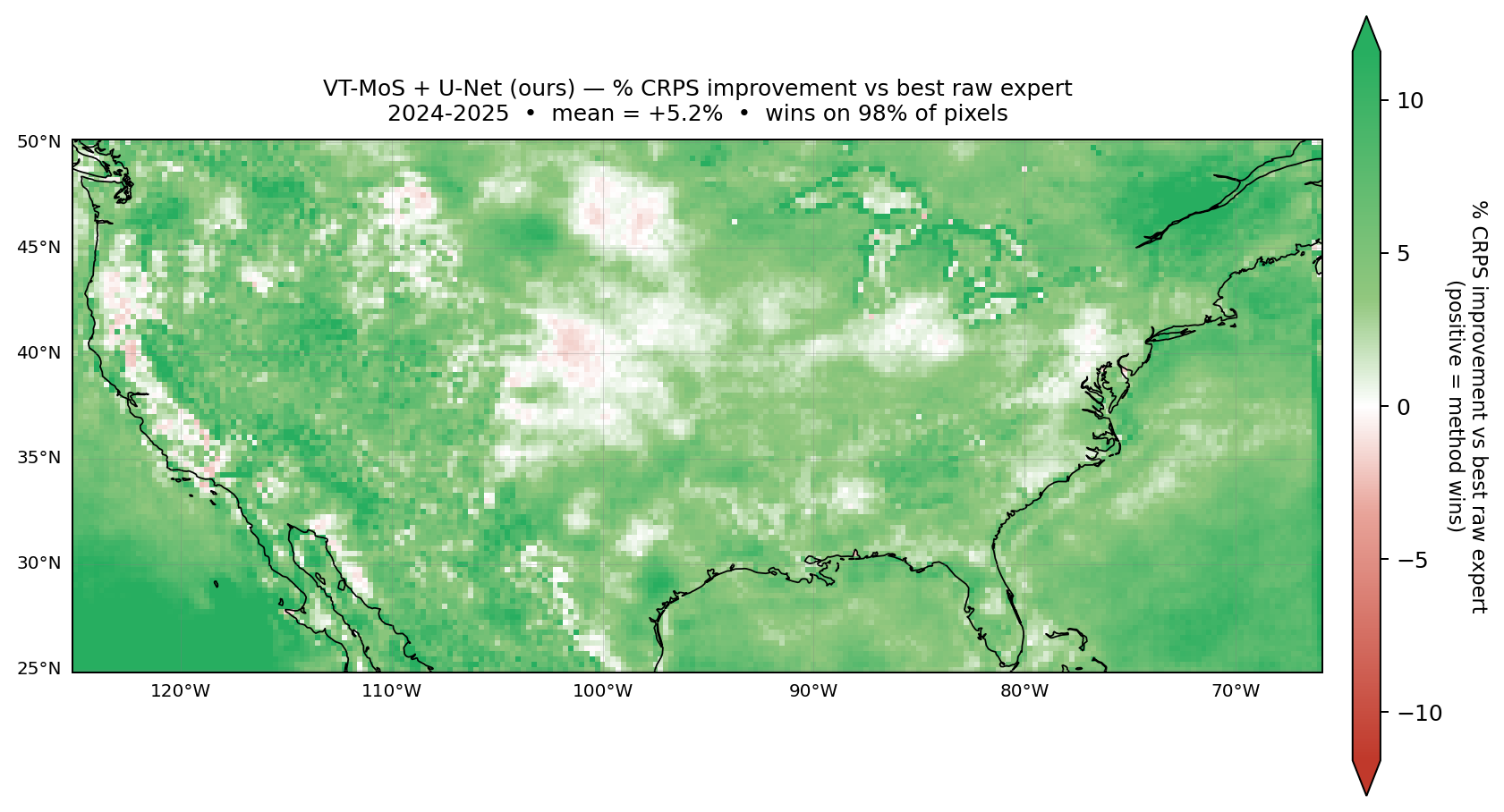}
  \caption{Per-pixel \% CRPS improvement of the hybrid method over the best raw expert, aggregated over 2024--2025.}
  \label{fig:usa-best-expert-spatial}
\end{figure}

\begin{figure}[!htbp]
  \centering

  \begin{subfigure}[t]{0.9\linewidth}
    \centering
    \includegraphics[width=\linewidth]{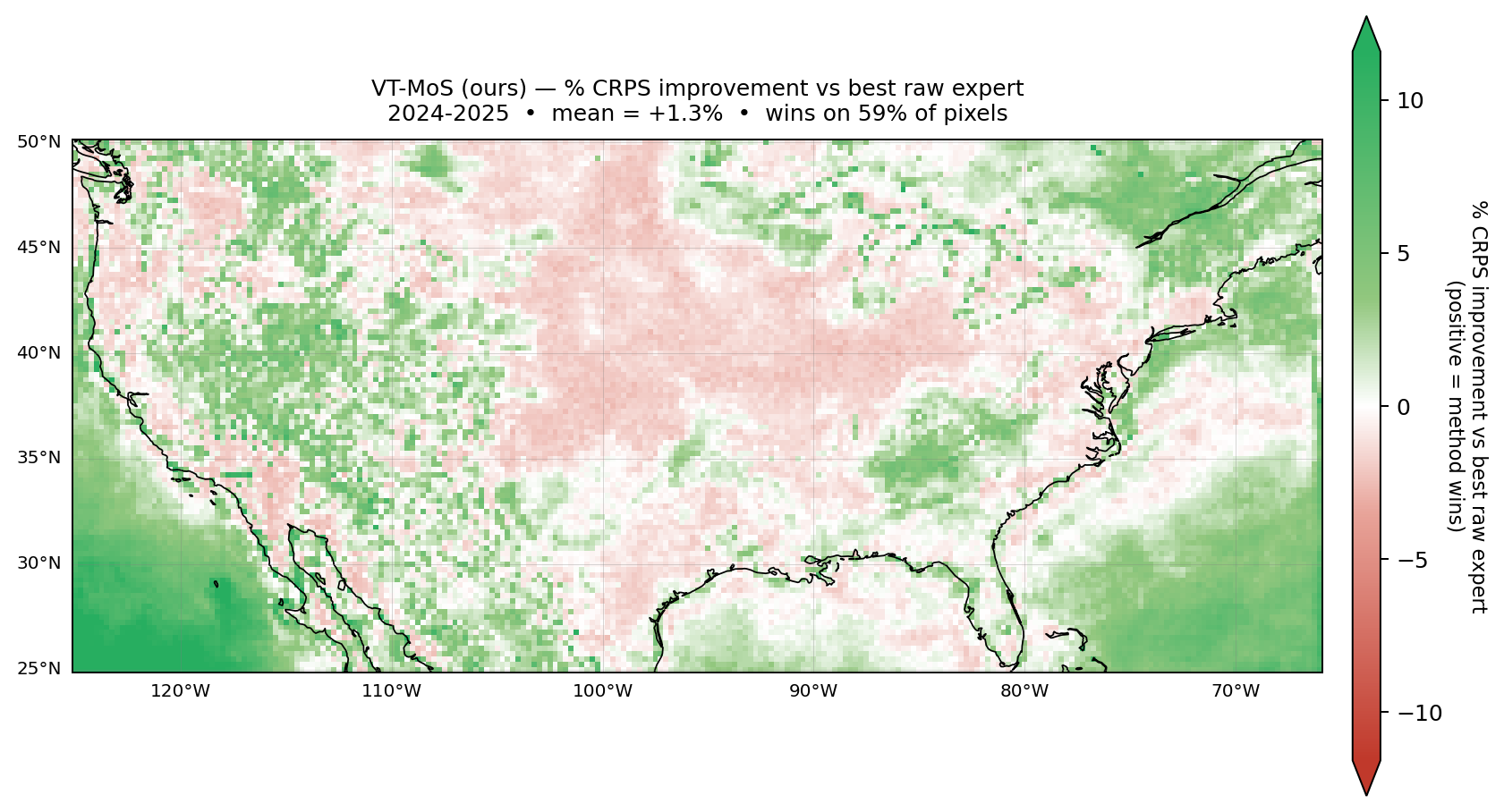}
    \caption{VT-MOS.}
    \label{fig:usa-vtmos-improvement}
  \end{subfigure}

  \vspace{1mm}

  \begin{subfigure}[t]{0.9\linewidth}
    \centering
    \includegraphics[width=\linewidth]{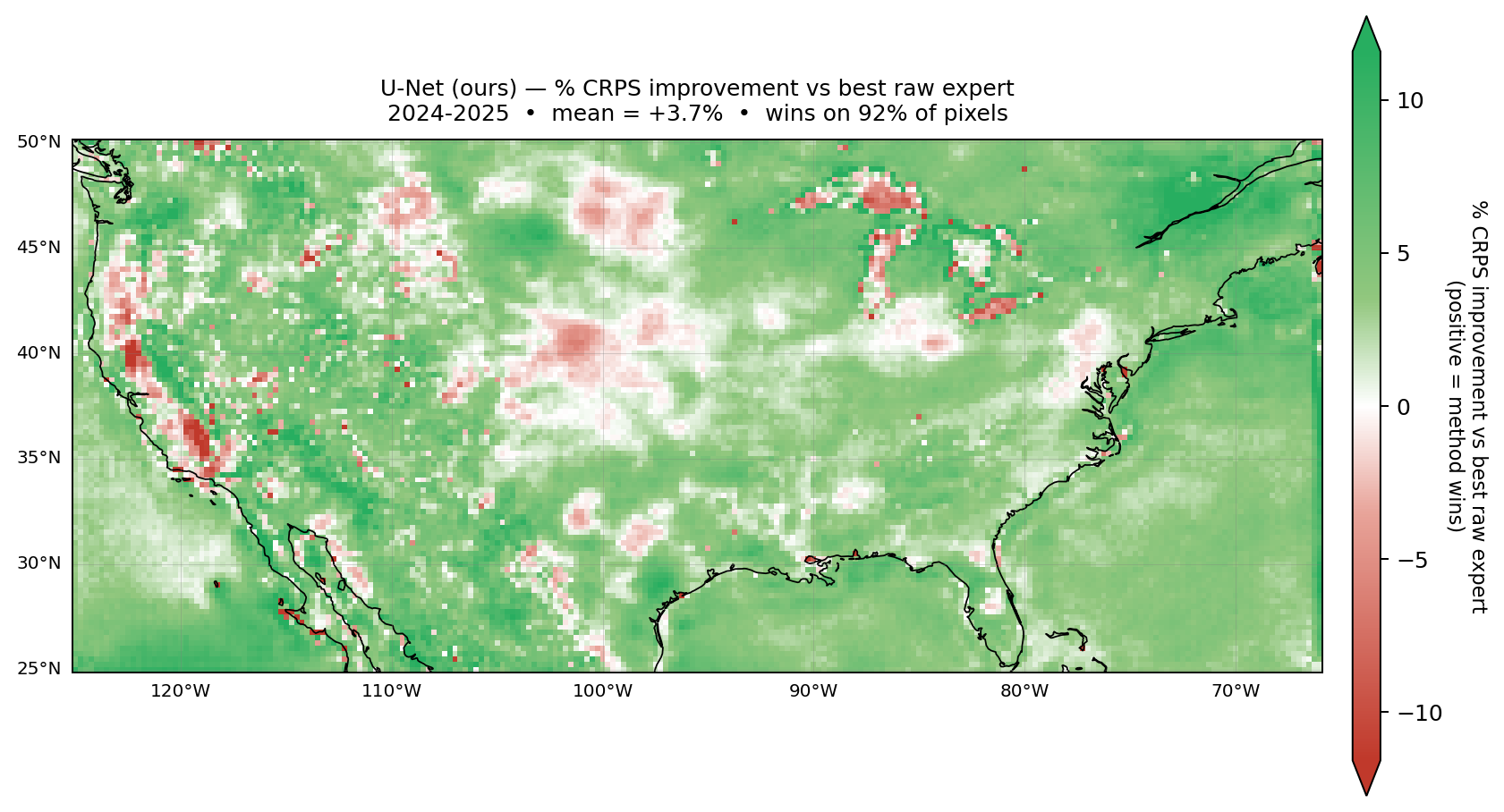}
    \caption{U-Net.}
    \label{fig:usa-vtmos-unet-improvement}
  \end{subfigure}

  \caption{Per-pixel \% CRPS improvement over the best raw expert, aggregated over 2024--2025. Positive values indicate that the combiner outperforms the best raw expert.}
  \label{fig:usa-avg-improvement-maps}
\end{figure}

\begin{figure}[!htbp]
  \centering
  \includegraphics[width=\linewidth]{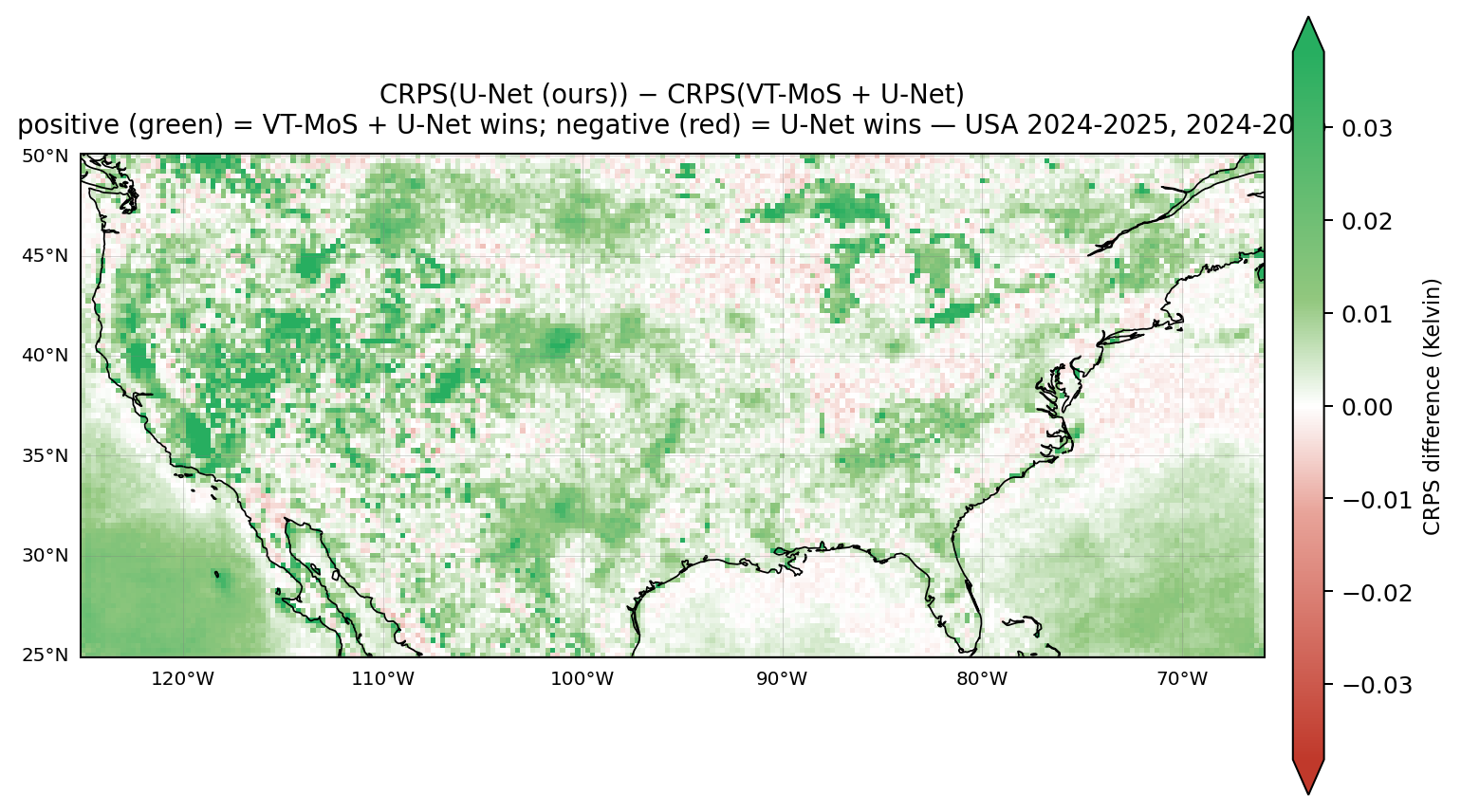}
  \caption{Difference map used to localise the BIH-evaluation patch. Green cells indicate regions where the hybrid model outperforms the offline U-Net.}
  \label{fig:usa-bih-region-diff}
\end{figure}

\begin{figure}[!htbp]
  \centering

  \begin{subfigure}[t]{0.9\linewidth}
    \centering
    \includegraphics[width=\linewidth]{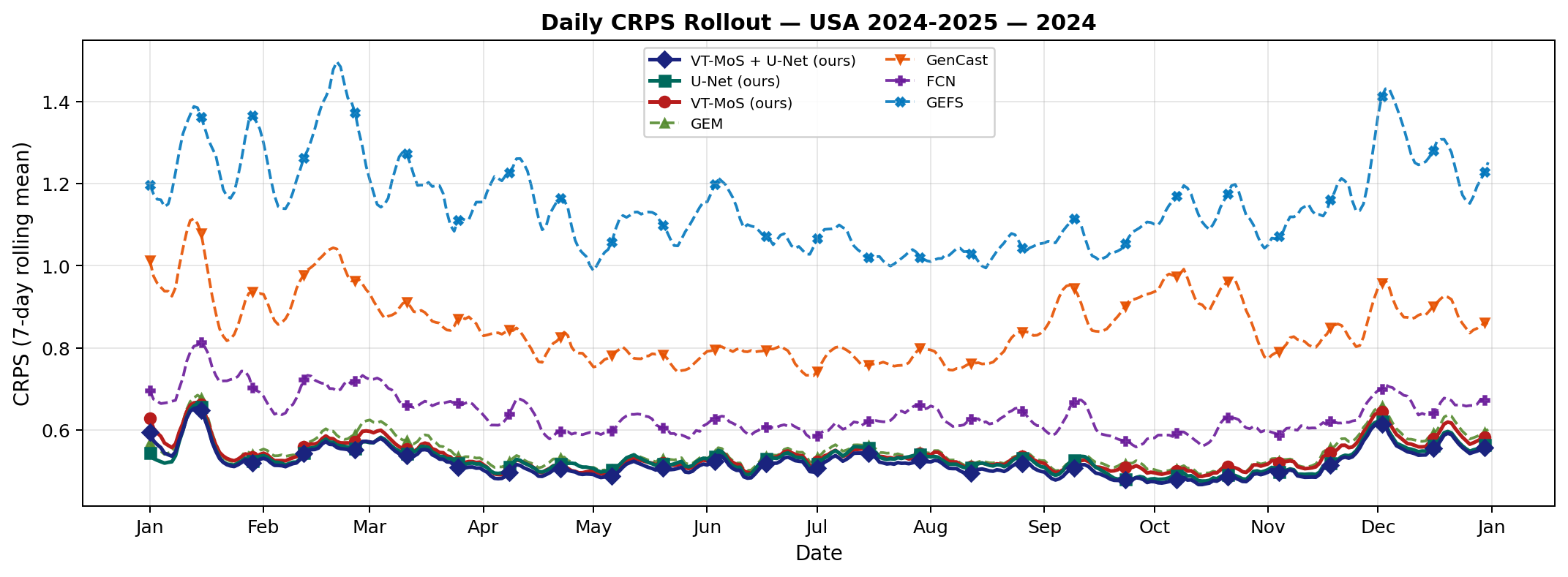}
    \caption{Rollout for 2024.}
    \label{fig:usa-rollout-2024}
  \end{subfigure}
  \hfill
  \begin{subfigure}[t]{0.9\linewidth}
    \centering
    \includegraphics[width=\linewidth]{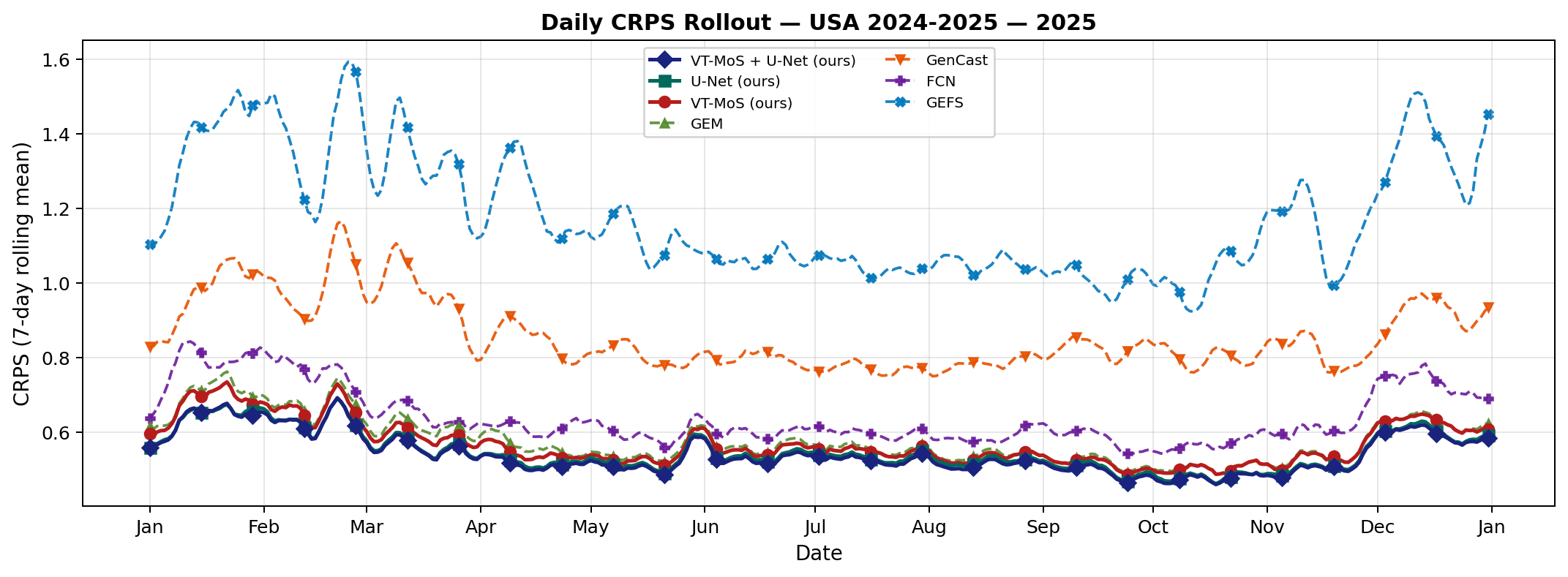}
    \caption{Rollout for 2025.}
    \label{fig:usa-rollout-2025}
  \end{subfigure}

  \vspace{1mm}

  \begin{subfigure}[t]{0.9\linewidth}
    \centering
    \includegraphics[width=\linewidth]{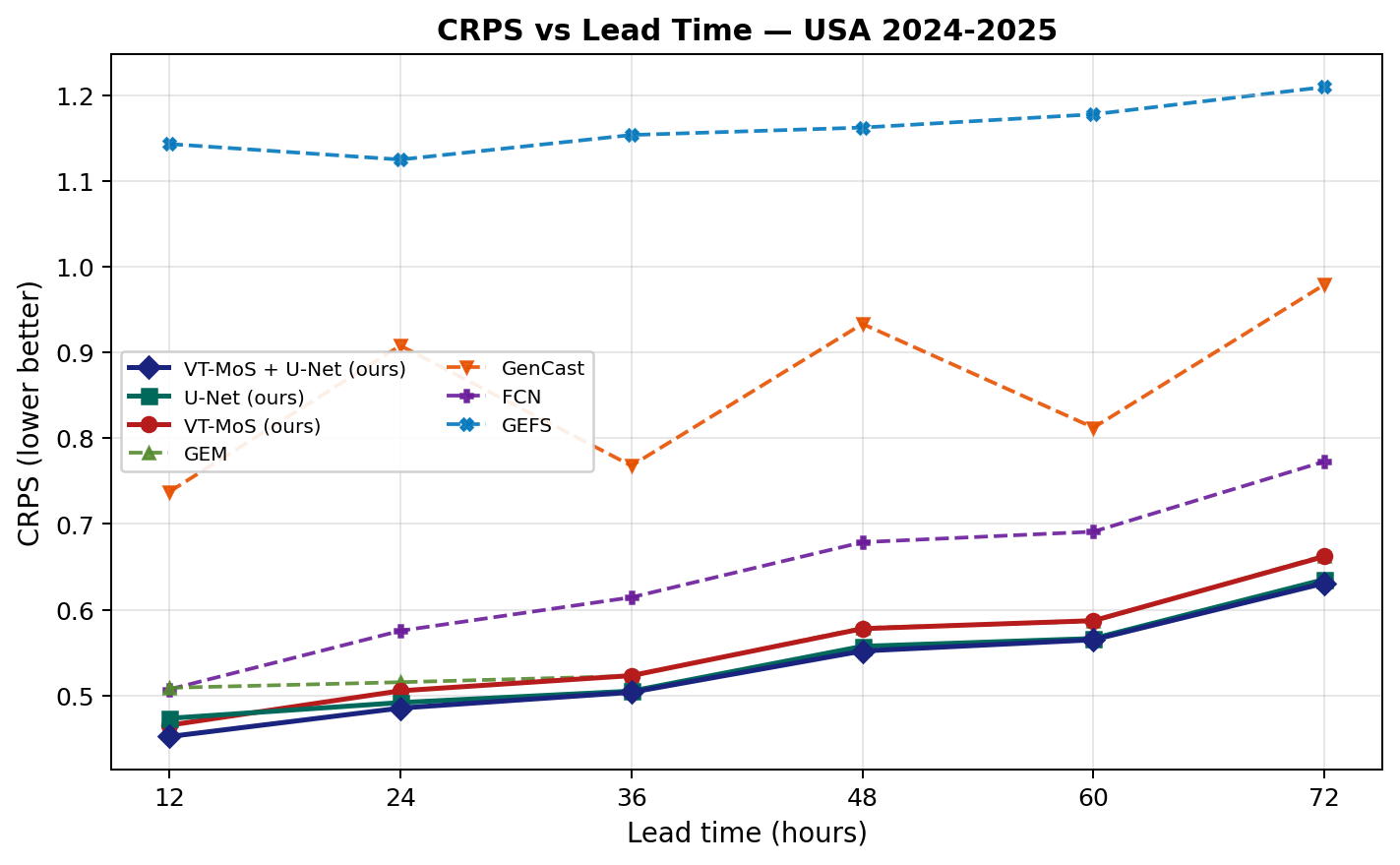}
    \caption{Per-lead CRPS.}
    \label{fig:usa-per-lead-crps}
  \end{subfigure}

  \caption{Rollouts and per-lead CRPS aggregated over 2024--2025.}
  \label{fig:usa-rollouts-per-lead}
\end{figure}

\FloatBarrier

\FloatBarrier
\section{Background: Prediction with Expert Advice}
\label{sec:backgroundexperts}

Prediction with Expert Advice (PWEA) is a foundational framework in online learning and sequential decision making. In this setting, a learner repeatedly combines predictions from a set of experts and tries to perform nearly as well as the best expert in hindsight. Unlike classical statistical learning, the framework does not require assumptions about the underlying data distribution and is therefore well suited for non-stationary or adversarial environments.

Formally, let there be $N$ experts. At each round $t \in \{1, \dots, T\}$, each expert $i$ produces a prediction $\hat{y}_{t,i}$. The learner assigns a weight vector
\[
w_t = (w_{t,1}, \dots, w_{t,N}),
\]
where $w_{t,i} \geq 0$ and $\sum_{i=1}^{N} w_{t,i} = 1$, and forms a combined prediction. After observing the true outcome $y_t$, a loss function $\ell(\cdot,\cdot)$ is incurred.

The performance of the learner is typically evaluated using \emph{regret}, defined as the difference between the cumulative loss of the learner and that of the best expert in hindsight:
\[
R_T
=
\sum_{t=1}^{T} \ell_t(\hat{y}_t)
-
\min_{i \in \{1,\dots,N\}}
\sum_{t=1}^{T} \ell_t(\hat{y}_{t,i}).
\]
A learning algorithm is considered successful if its regret grows sublinearly with time, i.e.,
\[
R_T = o(T),
\]
which implies that the average regret vanishes asymptotically.

One of the most influential algorithms in this framework is the Exponentially Weighted Average (EWA) forecaster, also known as Hedge. The method updates expert weights according to their past losses:
\[
w_{t+1,i}
=
\frac{
w_{t,i}\exp(-\eta \ell_t(\hat{y}_{t,i}))
}{
\sum_{j=1}^{N}
w_{t,j}\exp(-\eta \ell_t(\hat{y}_{t,j}))
},
\]
where $\eta > 0$ is the learning rate. Experts with lower historical loss receive larger weights over time.

A central concept in the theory is that of \emph{mixable losses}. For mixable loss functions, algorithms based on exponential weighting admit strong logarithmic regret guarantees of the form
\[
R_T \leq \frac{\log N}{\eta},
\]
which are independent of the time horizon $T$. Common examples include the logarithmic loss and squared loss under suitable assumptions.

Prediction with Expert Advice has been widely applied in ensemble forecasting, finance, adaptive control, and weather prediction, where multiple models exhibit varying performance across spatial and temporal regimes. The framework is particularly attractive in operational forecasting settings because it enables adaptive combination of heterogeneous models while retaining theoretical guarantees on long-term performance.

\subsection{Vovk's aggregation algorithm}
\label{subsubsec:vovk}

Vovk's exponentially weighted forecaster \citep{vovk1990aggregating}
maintains weights $w_t \in \Theta$, initialised to the uniform vector
$w_1 = (1/N, \dots, 1/N)$. At each round it outputs the linear mixture
$\hat F_t = \sum_n (w_t)_n F_{n,t}$, then updates
\[
  (w_{t+1})_n
  \;=\; \frac{(w_t)_n\, e^{-\CRPS(F_{n,t}, y_t)}}
             {\sum_{m} (w_t)_m\, e^{-\CRPS(F_{m,t}, y_t)}}.
\]
See Algorithm~\ref{alg:ewf}.

\begin{algorithm}[H]
\caption{Vovk's aggregation algorithm}
\label{alg:ewf}
\begin{algorithmic}[1]
\Require Number of experts $N$, support bounds $a < b$.
\State Initialise $(w_1)_n \gets 1/N$ for $n \in [N]$.
\For{$t = 1, 2, \ldots, T$}
  \State Receive expert predictions $\{F_{n,t}\}_{n=1}^{N}$.
  \State \textbf{Output} $\hat F_t = \sum_{n=1}^{N} (w_t)_n F_{n,t}$.
  \State Receive observation $y_t$.
  \State Update
    $(w_{t+1})_n \gets
      \dfrac{(w_t)_n e^{-\CRPS(F_{n,t}, y_t)}}
            {\sum_{m} (w_t)_m e^{-\CRPS(F_{m,t}, y_t)}}$
    for each $n \in [N]$.
\EndFor
\end{algorithmic}
\end{algorithm}

For the VT-MOS algorithm, the integral over
$\Theta=\Delta^{N-1}$
is approximated using Monte Carlo integration with
$M$ simplex samples drawn once and reused across rounds
(Common Random Numbers).

\section{Proof of Mixability}
\label{app:mixproof}

In this section, we furnish the complete proof of Theorem~\ref{thm:mixability}. In order to do this, we require the following result which is informally stated in \cite{freund1996predicting}, and is a continuous version of the analogous result in \cite{haussler1994tight}. The proof is a direct calculation which we will omit.

\begin{proposition}\label{prop:freund}
    Consider the loss function $l':\{0,1\} \times [0,1] \to [0,1]$ given by $l'(\omega,y) = (\omega - y)^2$. Let $\pi$ be any finite measure on $[0,1]$. Then, there exists $\phi \in [0,1]$ such that $$
e^{-2l'(\phi, y)} \leq \frac{\int_{0}^1 e^{-2l'(t,y)} d\pi(t)}{\pi([0,1])} \text{ for $y \in\{0,1\}$}.
    $$
An explicit formula for $\phi$ is given by $$
\phi = \frac 12 - \frac 14 \ln \left(\frac{\int_{0}^1 e^{-2l'(t,0)} d\pi(t)}{\int_{0}^1 e^{-2l'(t,1)} d\pi(t)}\right).
$$
\end{proposition}

We are now ready to begin the proof.

\begin{proof}[Proof of Theorem~\ref{thm:mixability}]

We shall borrow several notations and arguments from the proof of \cite[Theorem 2]{vyugin2021online}. 

Suppose that $F_n, n\in [N]$ are expert predictions and $\nu$ is a finite measure over $\Theta$. 

For any CDF $F$ of a random variable supported on $\mathcal{X}$ and $d \in \mathbb N$, define $z_k = \frac{k}{d}a + \frac{d-k}{d}b$ for $k=0,\ldots,d$. Note that $z_0 = a, z_d = b$ and $z_k$ are equally spaced in between $b-a$ having distance $\Delta = \frac{b-a}{d}$ between consecutive points. Now, define the CDF $F_d$ by

\begin{equation}\label{Fd}
F_d(z) = \begin{cases}
0 & z \leq a \\
1 & z \geq b \\
F\left(z_k\right) & z\in \left[z_{k-1},z_k\right], 1 \leq k \leq d. 
\end{cases}
\end{equation}

One may consider $F_d$ as a "discretization" of $F$. Now, we can derive the bound below exactly as in \cite[equation (21)]{vyugin2021online}. 
\begin{equation}\label{bd1}
|\CRPS(F_d,y) - \CRPS(F,y)| \leq 2\Delta
\end{equation}
for all $y\in \mathbb{R}$.

We will perform yet another simplification. For $y \in \mathbb{R}$ and $s \in [d]$, let $\omega_{y,s} = 1_{\{z_s \geq y\}} \in \{0,1\}$ and $\om_y =(\omega_{y,1},\ldots,\omega_{y,d})$.
Roughly, $\om_y$ is a discretization of the CDF $1_{\{z \leq y\}}$. 

Once again, exactly as in \cite[equation (22)]{vyugin2021online}, we obtain 
\begin{equation*}
\left|\CRPS(F_d,y) - \Delta \sum_{s=1}^d (F(z_s) - \omega_{y,s})^2\right| \leq \Delta
\end{equation*}
for all $y \in \mathbb R$. Combining this with \eqref{bd1} we have 
\begin{equation*}
    \left| \Delta \sum_{s=1}^d (F(z_s) - \omega_{y,s})^2 - \CRPS(F,y)\right| \leq 3\Delta
\end{equation*}
for all $y \in \mathbb R$. Consequently, for any CDF $F$ supported on $[a,b]$ and $y \in \mathbb{R}$, \begin{equation}\label{conv}
\lim_{d \to \infty} \Delta \sum_{s=1}^d (F(z_s) - \omega_{y,s})^2 = \CRPS(F,y).
\end{equation}

Having completed this approximation, we return to the expert predictions $F_n, n \in [N]$ and measure $\nu$ over $\Theta$. For ease, we let $p = (p_1,\ldots,p_N) \in \Theta$ denote a probability vector in the simplex.

Recall that our aim is to find a CDF $\hat{F}$ such that 
$$
e^{-\frac{2}{b-a}\CRPS(\hat{F},y)} \geq \frac{\int_{\Theta} e^{-\frac{2}{b-a} \CRPS\left(\sum_{n=1}^N p_nF_{n},y\right)} d\nu(p)}{\nu(\Theta)} \text{ for all $y\in \mathbb{R}$}.
$$

Fix $1 \leq s \leq d$. Define the measure $\pi_s$ on $[0,1]$ by \begin{equation}\label{pis}
\pi_s(A) = \frac{\int_{\Theta} 1_{\{\sum_{n=1}^N p_nF_n(z_s) \in A\}} d\nu(p)}{\nu(\Theta)}.
\end{equation}
Observe that $\pi_s([0,1]) = 1$. Now, there exists $\hat{f}_s \in [0,1]$ such that \begin{equation}\label{freund1}
e^{-2l'(\hat{f}_s,\Omega_s)} \geq \int_0^1 e^{-2l'(t,\Omega_s)} d \pi_s(t)
\end{equation}
for all $\Omega_s \in \{0,1\}$, where $\hat{f}_s = \frac 12 - \frac 14 \ln\left(\frac{\int_0^1 e^{-2l'(t,0)} d\pi_s(t)}{\int_0^1 e^{-2l'(t,1)} d\pi_s(t)}\right)$
by Proposition~\ref{prop:freund}. By the definition \eqref{pis} of $\pi_s$ and Fubini's theorem we have 
\begin{gather}\int_0^1 e^{-2l'(t,\Omega_s)} d \pi_s(t) = \frac{\int_{\Theta} e^{-2l'(\sum_{n=1}^N p_nF_n(z_s), \Omega_s)} d\nu(p)}{\nu(\Theta)}\nonumber
\\
\hat{f}_s = \frac 12 - \frac 14 \ln\left(\frac{\int_{\Theta} e^{-2l'(\sum_{n=1}^N p_nF_n(z_s), 0)} d\nu(p)}{\int_{\Theta} e^{-2l'(\sum_{n=1}^N p_nF_n(z_s), 1)} d\nu(p)}\right).\label{later}
\end{gather}
Combining the first of these with \eqref{freund1},
$$
e^{-2l'(\hat{f}_s,\Omega_s)} \geq \frac{\int_{\Theta} e^{-2l'(\sum_{n=1}^N p_nF_n(z_s), \Omega_s)} d\nu(p)}{\nu(\Theta)}.
$$
Multiplying this estimate over $s\in [d]$ and taking both sides to the power $\Delta$ gives \begin{equation}\label{bd3}
e^{-2\Delta \sum_{s=1}^d l'(\hat{f}_s, \Omega_s)} \geq \prod_{s=1}^d \left(\frac{\int_{\Theta} e^{-2l'(\sum_{n=1}^N p_nF_n(z_s), \Omega_s)} d\nu(p)}{\nu(\Theta)}\right)^\Delta.
\end{equation}
Let $h_s(p) = e^{-2\Delta l'(\sum_{n=1}^{N} {p_nF_n(z_s)},\Omega_s)}$ for $p \in \Theta$. Observe that by H\"older's inequality, 
$$
\prod_{s=1}^d \|h_s\|_{\frac 1\Delta} \geq \left\|\prod_{s=1}^d h_s\right\|_{\frac 1{\Delta d}}.
$$
Here, all norms are with respect to the measure $\nu$ on $\Theta$. Therefore, recalling the definition of $h_s$ we have 
\begin{multline*}
 \prod_{s=1}^d \left(\frac{\int_{\Theta} e^{-2l'(\sum_{n=1}^N p_nF_n(z_s), \Omega_s)} d\nu(p)}{\nu(\Theta)}\right)^\Delta \\ \geq \left(\frac{\int_{\Theta} e^{-\frac{2}{d}\sum_{s=1}^d l'(\sum_{n=1}^N p_nF_n(z_s), \Omega_s)} d\nu(p)}{\nu(\Theta)}\right)^{\Delta d}.
\end{multline*}
Combining this with \eqref{bd3} and noting that $\Delta d = b-a$ we have 
\begin{equation*}
e^{-2\Delta \sum_{s=1}^d l'(\hat{f}_s, \Omega_s)} \geq \left(\frac{\int_{\Theta} e^{-\frac{2}{b-a}\Delta \sum_{s=1}^d l'(\sum_{n=1}^N p_nF_n(z_s), \Omega_s)} d\nu(p)}{\nu(\Theta)}\right)^{b-a}.
\end{equation*}
A simple rearrangement yields the bound
\begin{equation}\label{bd4}
e^{-\frac{2}{b-a}\Delta \sum_{s=1}^d l'(\hat{f}_s, \Omega_s)} \geq \frac{\int_{\Theta} e^{-\frac{2}{b-a}\Delta \sum_{s=1}^d l'(\sum_{n=1}^N p_nF_n(z_s), \Omega_s)} d\nu(p)}{\nu(\Theta)}.
\end{equation}

Finally, let $y \in \mathbb{R}$ be arbitrary. Apply the above estimate to $\Omega_s = (\omega_{y,s})$ to obtain 
\begin{equation}\label{bd5}
e^{-\frac{2}{b-a}\Delta \sum_{s=1}^d (\hat{f}_{s}- \omega_{y,s})^2} \geq \frac{\int_{\Theta} e^{-\frac{2}{b-a}\Delta \sum_{s=1}^d l'(\sum_{n=1}^N p_nF_n(z_s), \omega_{y,s})} d\nu(p)}{\nu(\Theta)},
\end{equation}
where by \eqref{later} and the definition of $l'$, 
\begin{align*}
\hat{f}_{s} &=  \frac 12 - \frac 14 \ln\left(\frac{\int_{\Theta} e^{-2l'(\sum_{n=1}^N p_nF_n(z_s), 0)} d\nu(p)}{\int_{\Theta} e^{-2l'(\sum_{n=1}^N p_nF_n(z_s), 1)} d\nu(p)}\right)\\
&=  \frac 12 - \frac 14 \ln\left(\frac{\int_{\Theta} e^{-2[\sum_{n=1}^N p_nF_n(z_s)]^2} d\nu(p)}{\int_{\Theta} e^{-2[1-\sum_{n=1}^N p_nF_n(z_s)]^2} d\nu(p)}\right).
\end{align*}
In particular, we have that $\hat{f}_s = \hat{F}(z_s)$ where 
$$
\hat{F}(y) =  \frac 12 - \frac 14 \ln\left(\frac{\int_{\Theta} e^{-2[\sum_{n=1}^N p_nF_n(y)]^2} d\nu(p)}{\int_{\Theta} e^{-2[1-\sum_{n=1}^N p_nF_n(y)]^2} d\nu(p)}\right)
$$
is a well-defined CDF. Finally, letting $d \to \infty$ in \eqref{bd5}, by Fubini's theorem and \eqref{conv} we obtain $$
e^{-\frac{2}{b-a} \CRPS(\hat{F},y)} \geq \frac{\int_{\Theta} e^{-\frac{2}{b-a} \CRPS\left(\sum_{n=1}^N p_nF_{n},y\right)} d\nu(p)}{\nu(\Theta)}.
$$
This completes the proof.
\end{proof}

\section{Proof of Regret Bound}
\label{app:regproof}

In this section, we prove the regret bound Theorem~\ref{thm:reg}. First, we shall assume the truth of Lemmas~\ref{lem:quadratic} and \ref{bds}, and prove the theorem. Subsequently, we shall prove these additional lemmas.

\begin{proof}[Proof of Theorem~\ref{thm:reg}]
   Let $p^* \in \argmin_{\Theta} \sum_{t=1}^T \CRPS\left(\sum_{n=1}^N p_nF_{n,t}, y_t\right)$. This exists, since we are minimizing a continuous function over a compact set. 

   Iterating the equation \eqref{eq:update} over $t\in [T]$ and applying Lemma~\ref{lem:quadratic} we obtain $$
\mu_{T+1}(\Theta) = \int_{\Theta} e^{-\frac{2}{b-a} \sum_{t=1}^T \CRPS\left(\sum_{n=1}^N p_nF_{n,t}, y_t\right)}dp = \int_{\Theta} e^{-\frac{2}{b-a}(A^T p - p^tBp)}dp,
   $$
   where the integration is now performed with respect to the $(N-1)$-dimensional Lebesgue measure on $\Theta$. Recall that $p\in \Theta$ is written as $p = (p_1,\ldots,p_N)$.
   
   Now, by \eqref{eq:lossmin} and Lemma~\ref{lem:quadratic},
   \begin{align*}
   \Reg_{M,T} &\leq -\frac{b-a}{2} \ln \mu_{T+1}(\Theta) - \sum_{t=1}^T \CRPS\left(\sum_{n=1}^N p^*_nF_{n,t}, y_t\right)\\
   & = -\frac{b-a}{2} \ln \int_{\Theta} e^{-\frac{2}{b-a}(A^Tp - \frac 12 p^TBp)} - \left[A^Tp^* - \frac 12 (p^*)^TBp^*\right] \\
   & = -\frac{b-a}{2} \ln \int_{\Theta} e^{-\frac{2}{b-a} (A^Tp - \frac 12 p^TBp)} dp -  \frac {b-a}2 \ln e^{\frac{2}{b-a} (A^Tp^* - \frac 12 (p^*)^TBp^*)}\\
   &= -\frac{b-a}{2}\ln \int_{\Theta} e^{\frac{2}{b-a}\left[A^{T}p^* -\frac 12 (p^*)^TB p^*   - (A^{T}p -  \frac 12 p^TBp)\right]} dp.
   \end{align*}
Let $L(p) = \sum_{t=1}^T \CRPS(\sum_{n=1}^N p_nF_{n,t},y_t)$. Then, we have shown that 
\begin{equation}\label{regbd}
\Reg_{M,T} \leq - \frac{b-a}{2}\ln \int_{\Theta} e^{-\frac{2}{b-a} (L(p_1,\ldots,p_N)-L(p^*))}dp.
\end{equation}

We shall now further simplify the exponent. To observe this, we shall simply treat $L$ as being expanded in a Taylor series about $p^*$. Note that $\nabla L(p^*) = A - Bp^*$, while the Hessian $\nabla^2 L(p^*) = -B$. Since $L$ is quadratic, there are no further terms to consider. Therefore, if $\delta = (p-p^*)$, then the Taylor expansion yields
\begin{equation}\label{taylor}
L(p) - L(p^*) = (A-Bp^*)^T\delta -\frac 12\delta^T B\delta.
\end{equation}
Observe that we do not assume anything about the point $p^*$ at all during this manipulation: indeed, if $p^*$ is interior to $\Theta$ then $(A-Bp^*) = 0$ since $p^*$ is an interior local minima of $L$ in $\Theta$. However, if $p^*$ is on the boundary, then we cannot assume that the derivative is zero at $p$, since $L$ may decrease in directions away from $\Theta$ at $p$. Note that $\mathbf{1}^T \delta = 0$ where $\mathbf{1}$ is a vector all of whose entries are $1$.

We can now attempt to bound the right hand side of \eqref{taylor}. By Lemma~\ref{bds}, \begin{align}
(A-Bp^*)^T \delta = & (A-Bp^* + \min_{i}(A - Bp^*)_i\mathbf{1})^T \delta \nonumber\\ \leq & \max_{1 \leq i,j \leq N} |(A_i - (Bp^*)_i) - (A_j - (Bp^*)_j)| \times \|\delta\|_1  \nonumber\\
\leq & \left[\max_{1 \leq i,j \leq N} |A_i-A_j| + \max_{1 \leq i,j \leq n}  |(Bp^*)_i - (Bp^*)_j|\right] \times \|\delta\|_1 \nonumber\\
\leq & \left[2(b-a)T + (b-a)T\right] \times \|\delta\|_1 \nonumber\\
\leq & 3(b-a)T \|\delta\|_1 \leq 3(b-a)T\sqrt{N} \|\delta\|_2,\nonumber
\end{align}
where we used the Cauchy-Schwarz inequality in the last line. For the other term, we simply have $$
-\frac 12 \delta^T B \delta \leq (b-a)NT \|\delta\|_2^2
$$
by Lemma~\ref{bds}. Combining the above two bounds and \eqref{taylor},\begin{equation}\label{Lbd}
L(p) - L(p^*) \leq 3(b-a)T \sqrt{N} \|\delta\|_2 + (b-a)NT \|\delta\|_2^2 \leq (b-a)T(3\sqrt{N}+N) \|\delta\|_2
\end{equation}
since $\|\delta_2\| \leq 1$. Now, let $r = \frac{1}{T(3\sqrt{N}+N)}$. Then, if $\|\delta\|_2 \leq r$, applying the definition of $L$ and \eqref{Lbd} implies that $L(p) - L(p^*) \leq b-a$. In particular, 
$$
\|p^*-p\| \leq r \implies e^{-\frac{2}{b-a}(L(p) -L(p^*))} \geq e^{-2}.
$$
Finally, let $S = \Theta \cap \{\|p^* - p\| \leq r\}$. Then, \begin{equation*}\int_{\Theta} e^{-\frac{2}{b-a} (L(p_1,\ldots,p_N)-L(p^*))}dp \geq e^{-2} |S|,
\end{equation*}
where $|S|$ is the $(N-1)$-dimensional Lebesgue measure of $S$. This combined with \eqref{regbd} yields
\begin{equation}\label{l2l}
\Reg_{M,T} \leq - \frac{b-a}{2}\ln e^{-2}|S| \leq \frac{b-a}{2}\ln \frac 1{|S|}+(b-a).
\end{equation}
It remains to bound $|S|$. However, a property of the simplex $\Theta$ is that for some universal constant $c_{\Theta}$ depending upon only $n$, $|\Theta \cap \mathbf{B}(q,s)| \geq c_{\Theta} |\mathbf{B}(q,s)|$ for all $q\in \Theta, s >0$, where $\mathbf{B}(q,s) = \{q': \|q-q'\| \leq s\}$. Thus, $$
\ln \frac{1}{|S|} \leq \ln \frac{1}{c_{\Theta} |\mathbf{B}(p^*,r)|} = (N-1)\ln \frac 1r + C = (N-1)\ln(T)+ C
$$
where $C$ is independent of $T$. Finally, combining this with \eqref{l2l} gives $$
\Reg_{M,T} \leq \frac{(b-a)(N-1)}{2}\ln(T) + C
$$
where $C$ has no dependence on $T$, completing the proof.
\end{proof}

Now, we shall prove the lemmas.

\begin{proof}[Proof of Lemma~\ref{lem:quadratic}]

Let $t\in [T]$, $p\in \Theta$ and $x,y \in [a,b]$ be arbitrary. Let $F_{n,t}, n \in [N]$ denote the expert predictions at time $t$. We have 
\begin{align*}
(F_t^{(p)}(x) - 1_{\{x \leq y\}})^2 &= \left(\sum_{n=1}^N p_nF_{n,t}(x) - 1_{\{x \leq y\}}\right)^2\\
&= \left(\sum_{n=1}^N p_nF_{n,t}(x)\right)^2 - 2\  1_{\{x \leq y\}} \sum_{n=1}^N p_nF_{n,t}(x) + 1_{\{x \leq y\}}.
\end{align*}

On the other hand, \begin{align*}
    \sum_{n=1}^N p_n (F_{n,t}(x) - 1_{\{x \leq y\}})^2 = \sum_{n=1}^N p_n F_{n,t}(x)^2 + 1_{\{x \leq y\}} -2 \sum_{n=1}^N p_n F_{n,t}(x) 1_{\{x \leq y\}}.
\end{align*}

Subtracting the previous equation from this one, $$
 \sum_{n=1}^N p_n (F_{n,t}(x) - 1_{\{x \leq y\}})^2 - (F_t^{(p)}(x) - 1_{\{x \leq y\}})^2 = \sum_{n=1}^N p_n F_{n,t}(x)^2 -\left(\sum_{n=1}^N p_nF_{n,t}(x)\right)^2.
$$
However, we note the algebraic identity $$
 \sum_{n=1}^N p_n F_{n,t}(x)^2 -\left(\sum_{n=1}^N p_nF_{n,t}(x)\right)^2 = \sum_{n,m=1}^N p_mp_n(F_{n,t}(x) - F_{m,t}(x))^2.
$$
Combining this with the previous equality,
$$
(F_t^{(p)}(x) - 1_{\{x \leq y\}})^2 = \sum_{n=1}^N p_n (F_{n,t}(x) - 1_{\{x \leq y\}})^2 -  \sum_{n,m=1}^N p_mp_n(F_{n,t}(x) - F_{m,t}(x))^2.
$$
Summing this identity over $t \in [T]$ and then integrating over $x \in \R$, the identity immediately follows once we recall the definitions of $A$ and $B$ from the statement of the lemma.
\end{proof}

Next, we prove Lemma~\ref{bds}.

\begin{proof}[Proof of Lemma~\ref{bds}]
Note that if $F,G$ are CDFs of random variables concentrated on $[a,b]$, then \begin{equation}\label{trying}
\int_{a}^b (F(x) -G(x))^2 dx \leq \int_{a}^b 1 dx \leq (b-a).
\end{equation}
From here, part (1) follows directly from the triangle inequality $|A_i-A_j| \leq |A_i| + |A_j|$ and the definition of $A$. 

For the proof of part (2), let $p = (p_1,\ldots,p_N) \in \Theta$ be arbitrary. Then, for any $i,j \in [N]$, 
$$
|(Bp)_i - (Bp)_j| \leq \sum_{n=1}^N |B_{ni}p_i - B_{nj}p_j| \leq \max_{n,k\in [N]} |B_{nk}|
$$
since $p_i, p_j \in [0, 1]$ and $B$ has non-negative entries. However, this quantity is bounded by $T(b-a)$ by \eqref{trying}. The result follows.

Finally, for part (3) let $\delta = (\delta_1,\ldots,\delta_N)$ be an arbitrary vector orthogonal to $\mathbf{1}$ (so that its entries sum to $0$). One can verify that \begin{align*}
-\delta^T B \delta &= -\sum_{t=1}^T\sum_{i,j=1}^N \delta_i\delta_j \int_a^b (F_{i,t}(z) - F_{j,t}(z))^2 dz\\
&= 2\sum_{t=1}^T \int_a^b \left(\sum_{n=1}^N\delta_n F_{n,t}(z)\right)^2 dz.
\end{align*}
Following this, we note that $\left(\sum_{n=1}^N\delta_n F_{n,t}(z)\right)^2 \leq \left(\sum_{n=1}^N\delta_n)\right)^2 \leq \|\delta\|_2^2$ by the Cauchy-Schwarz inequality. Applying this to the previous inequality, we obtain $$
-\delta^T B \delta \leq 2\sum_{t=1}^T \int_a^b \|\delta\|_2^2 dz= 2(b-a)T \|\delta\|_2^2.
$$
Thus, the proofs of all three parts are complete.
\end{proof}

We end this section with a detailed comparison with \cite[Theorem 4]{freund1996predicting} \label{app:freund_recovery}. As mentioned in Section~\ref{sec:vtmos}, the optimal rate obtained by them is $\frac 14 \ln T$, while we obtain an extra factor of $2$ leading to the apparently non-optimal $\frac 12 \ln T$. 

This non-optimality arises from the term $(A-Bp^*)$ in \eqref{Lbd}. Indeed, were this term either zero or smaller in order, then $L(p) - L(p^*)$ could be bounded by $CT\|\delta\|_2$ where $C$ is independent of $T$, resulting in an improvement by a factor of $\frac 12$.

The naive expert assumption affords us this stronger bound as follows. Recall that $a=0,b=1, N=2$ and $F_{0,t} = 1_{\{x \leq 0\}}$ while $F_{1,t} = 1_{\{x \leq 1\}}$. Let $y_i, i \in [T]$ be the observations up till time $T$. One easily computes from \eqref{eq:crps} that $$
A = \left(\sum_{t=1}^T y_t, T - \sum_{t=1}^T y_t\right) 
$$
and $$
B = \begin{pmatrix}
    0 & T \\
    T & 0
\end{pmatrix}.
$$
We now compute $p^*$. Since we are in the setting of two experts, any probability vector is of the form $(1-q,q)$ for $q \in [0,1]$. Now, by Lemma~\ref{lem:quadratic} and the above formulas we have $$
L(1-q,q) = (1-2q)\sum_{t=1}^T y_t + Tq^2.
$$
Minimizing this over $q$ by taking the derivative yields $q^* = \frac{\sum_{t=1}^T y_t}{T}$, which is nothing but the empirical average of the observations $y_t$. 

Now, simple algebra yields $A-Bp^* = 0$. Following the proof above as usual does yield the optimal bound $\frac 14$ for \cite{freund1996predicting}, demonstrating that (up to the additive constant) we achieve minimax-optimal regret for naive two-expert predictions on $[0,1]$.

Another way of explaining this is as follows. Recall that in the proof of Theorem~\ref{thm:reg}, we were unable to set $(A-Bp^*) = 0$, since $p^*$ could lie on the boundary. However if $B$ is invertible and $B^{-1}A$ is a well-defined probability vector, then $p^* = B^{-1}A$ satisfies $(A-Bp^*) = 0$ and is a typical critical point of $L$, reducing the bound by a factor of $2$. This is precisely what happens in \cite{freund1996predicting} and would also be expected to occur in a number of situations, including when $p^*$ is guaranteed to strictly lie inside the simplex.

\end{document}